\documentclass[12pts]{article}
\usepackage[utf8]{inputenc}
\usepackage[round,authoryear]{natbib}
\usepackage[margin=1in]{geometry}
\usepackage{amsmath}
\usepackage{amssymb}
\usepackage{amsthm}
\usepackage{hyperref}
\usepackage{xcolor}
\definecolor{royalblue}{rgb}{0.0, 0.22, 0.66}
\definecolor{RoyalBlue}{cmyk}{1, 0.90, 0, 0}
\hypersetup{
    citecolor=RoyalBlue,
    colorlinks=true,
    linkcolor=RoyalBlue,
    filecolor=magenta,      
    urlcolor=teal,
}
\usepackage{thmtools}
\usepackage{thm-restate}
\usepackage{cleveref}
\usepackage{fancyhdr}
\usepackage[mathscr]{euscript}
\usepackage{bbm}
\usepackage{mathtools}
\usepackage{graphicx}
\usepackage{subcaption}
\usepackage{tikz}
\usetikzlibrary{positioning}
\usetikzlibrary{decorations.markings}
\def\Real{\mathop{\mathbb{R}}\nolimits}

\def\argmin{\mathop{\rm argmin}}
\def\argmax{\mathop{\rm argmax}}

\newcommand{\prob}{\mathbb{P}}
\newcommand{\one}{\mathbbm{1}}

\newcommand{\bi}{\boldsymbol{i}}

\newcommand{\bs}{\boldsymbol{s}}

\newcommand{\bv}{\boldsymbol{v}}

\newcommand{\bx}{\boldsymbol{x}}
\newcommand{\by}{\boldsymbol{y}}

\newcommand{\bmu}{\boldsymbol{\mu}}

\newcommand{\bX}{\boldsymbol{X}}

\newcommand{\bdelta}{\boldsymbol{\delta}}
\newcommand{\bepsilon}{\boldsymbol{\epsilon}}

\newcommand{\btheta}{\boldsymbol{\theta}}
\newcommand{\bomega}{\boldsymbol{\omega}}
\newcommand{\bxi}{\boldsymbol{\xi}}

\newcommand{\bupsilon}{\boldsymbol{\upsilon}}

\newcommand{\cB}{ \mathcal{B}}
\newcommand{\cC}{ \mathcal{C}}

\newcommand{\cF}{ \mathcal{F}}

\newcommand{\cH}{ \mathcal{H}}

\newcommand{\cL}{ \mathcal{L}}
\newcommand{\cM}{ \mathcal{M}}
\newcommand{\cN}{ \mathcal{N}}
\newcommand{\cO}{ \mathcal{O}}

\newcommand{\cR}{ \mathcal{R}}
\newcommand{\cS}{ \mathcal{S}}

\newcommand{\cW}{ \mathcal{W}}

\newcommand{\sF}{ \mathscr{F}}
\newcommand{\sG}{ \mathscr{G}}
\newcommand{\sH}{ \mathscr{H}}

\newcommand{\sN}{ \mathscr{N}}

\newcommand{\sP}{ \mathscr{P}}
\newcommand{\sQ}{ \mathscr{Q}}

\newcommand{\E}{\mathbb{E}}
\newcommand{\relu}{\text{ReLU}}

\newcommand{\fH}{\mathbb{H}}

\newcommand{\fL}{\mathbb{L}}
\newcommand{\fM}{\mathbb{M}}

\usepackage{titletoc}
\usepackage{wrapfig}

\newcommand{\vertiii}[1]{{\left\vert\kern-0.25ex\left\vert\kern-0.25ex\left\vert #1 
    \right\vert\kern-0.25ex\right\vert\kern-0.25ex\right\vert}}
    
\newtheorem{defn}{Definition}
\newtheorem{thm}[defn]{Theorem}
\newtheorem{lem}[defn]{Lemma}
\newtheorem{cor}[defn]{Corollary}
\newtheorem{ass}{A\hspace{-4pt}}

\usepackage[max2]{authblk}
\title{\vspace{-50pt}A Statistical Analysis for Supervised Deep Learning with Exponential Families for Intrinsically Low-dimensional Data}
\author[1]{Saptarshi Chakraborty\thanks{email: \texttt{saptarshic@berkeley.edu}}}
 \author[1,2,3]{Peter L.~Bartlett\thanks{email: \texttt{peter@berkeley.edu}}}
 \affil[1]{Department of Statistics, UC Berkeley}
  \affil[2]{Department of Electrical Engineering and Computer Sciences, UC Berkeley}
 \affil[3]{Google DeepMind}
\date{\vspace{-5ex}}

\begin{document}
\maketitle
\begin{abstract}
Recent advances have revealed that the rate of convergence of the expected test error in deep supervised learning decays as a function of the intrinsic dimension and \textit{not} the  dimension $d$ of the input space. Existing literature defines this intrinsic dimension as the Minkowski dimension or the manifold dimension of the support of the underlying probability measures, which often results in sub-optimal rates and unrealistic assumptions. In this paper, we consider supervised deep learning when the response given the explanatory variable is distributed according to an exponential family with a $\beta$-H\"older smooth mean function. We consider an entropic notion of the intrinsic data-dimension and demonstrate that with $n$ independent and identically distributed samples, the test error scales as $\tilde{\mathcal{O}}\left(n^{-\frac{2\beta}{2\beta + \bar{d}_{2\beta}(\lambda)}}\right)$, where $\bar{d}_{2\beta}(\lambda)$ is the $2\beta$-entropic dimension of $\lambda$, the distribution of the explanatory variables. This improves on the best-known rates. Furthermore, under the assumption of an upper-bounded density of the explanatory variables, we characterize the rate of convergence as $\tilde{\mathcal{O}}\left( d^{\frac{2\lfloor\beta\rfloor(\beta + d)}{2\beta + d}}n^{-\frac{2\beta}{2\beta + d}}\right)$, establishing that the dependence on $d$ is not exponential but at most polynomial. We also demonstrate that when the explanatory variable has a lower bounded density, this rate in terms of the number of data samples, is nearly optimal for learning the dependence structure for exponential families. 
\end{abstract}

\section{Introduction}
One hypothesis for the extraordinary performance of deep learning is that most natural data have an intrinsically low-dimensional structure despite lying in a high-dimensional feature space \citep{pope2020intrinsic}. Under this so-called \textit{``manifold hypothesis,"} the recent theoretical developments in the generalization aspects of deep learning theory literature have revealed that the excess risk for different deep learning models, especially regression \citep{schmidt2020nonparametric, suzuki2018adaptivity} and generative models \citep{huangjmlr,chakraborty2024a,chakraborty2024statistical} exhibit a decay pattern that depends only on the intrinsic dimension of the data. Notably, \citet{nakada}, \citet{huangjmlr} and \cite{chakraborty2024a} showed that the excess risk decays as $\cO(n^{-1/\cO(d_\mu)})$, where $d_\mu$ denotes the Minkowski dimension of the underlying distribution. For a supervised learning setting, this phenomenon has been proved for various deep regression models with additive Gaussian noise \citep{schmidt2020nonparametric, nakada, suzuki2018adaptivity, suzuki2021deep}.

Most of the aforementioned studies aim to describe this inherent dimensionality by utilizing the concept of the (upper) Minkowski dimension of the data's underlying support. However, it is worth noting that the Minkowski dimension primarily focuses on quantifying the rate of growth in the covering number of the support while neglecting to account for situations where the distribution may exhibit a higher concentration of mass within specific sub-regions of this support. Thus, the Minkowski dimension often overestimates the intrinsic dimension of the data distribution, resulting in slower rates of statistical convergence. On the other hand, some works \citep{chen2022nonparametric,chen2019efficient,jiao2021deep} try to impose a smooth Riemannian manifold structure for this support and characterize the rate through the dimension of this manifold. However, this assumption is not only very strong and difficult to verify, but also ignores the fact that the data can be concentrated only on some sub-regions and can be thinly spread over the remainder, again resulting in an over-estimate. In contrast, recent insights from the optimal transport literature have introduced the concept of the Wasserstein dimension \citep{weed2019sharp}, which overcomes these limitations and offers a more accurate characterization of convergence rates when estimating a distribution through the empirical measure. Furthermore, recent advancements in this field have led to the introduction of the entropic dimension \citep{chakraborty2024statistical}, which builds upon seminal work by \cite{dudley1969speed} and can be employed to describe the convergence rates for Bidirectional Generative Adversarial Networks (BiGANs) \citep{donahue2017adversarial}. Remarkably, the entropic dimension is no larger than the Wasserstein and Minkowski dimensions, resulting in faster convergence rates. However, it is not known whether this faster rate of convergence extends beyond Generative Adversarial Networks (GANs) and their variants to supervised learning problems.  

In this paper, we provide a statistical framework aimed at understanding deep supervised learning. Our approach involves modeling the conditional distribution of the response variable given the explanatory variable as a member of an exponential family with a smooth mean function. This framework accommodates a wide spectrum of scenarios, including standard regression and classification tasks, while also providing a statistical foundation for handling complex dependencies in real data settings. In this framework, the maximum likelihood estimates can be viewed as minimizing the canonical Bregman divergence loss between the predicted values and the actual responses. When the explanatory variable has a bounded density with respect to the $d$-dimensional Lebesgue measure, our analysis reveals that deep networks employing ReLU activation functions can achieve a test error on the order of $\Tilde{\cO}(n^{-2\beta/(2\beta +d)})$ provided that appropriately sized networks are selected. Here $\beta$ denotes the H\"{o}lder smoothness of the true mean response function. This generalizes the known results in the literature for additive regression with Gaussian noise.

Another aspect overlooked by the current literature is that even when the explanatory variable is absolutely continuous, the rate of convergence of the sample estimator often exponentially increases with the ambient feature dimension, making the upper bound on the estimation error weak for high-dimensional data. In this paper, we prove that if the explanatory variable has a bounded density, the dependence, in terms of the ambient feature dimension, is not exponential but at most polynomial.  Furthermore, we show that the derived rates for the test error are roughly minimax optimal, meaning that one cannot achieve a better rate of convergence through any estimator except for only potentially improving a logarithmic dependence on $n$. Lastly, when the data has an intrinsically low dimensional structure, we show that the test error can be improved to achieve a rate of roughly $\Tilde{\cO}(n^{-2\beta/(2\beta + \bar{d}_{2\beta}(\lambda))})$, where $\bar{d}_{2\beta}(\lambda)$ denotes the $2\beta$-entropic dimension (see Definition~\ref{ed}) of $\lambda$, the distribution of the explanatory variables, thus, improving upon the rates in the current literature. This result not only extends the framework beyond additive Gaussian noise regression models but also improves upon the existing rates \citep{nakada,schmidt2020nonparametric,chen2022nonparametric}.  The main results of this paper are summarized as follows:

\begin{itemize}
    \item In Theorem \ref{main_full}, we demonstrate that when the explanatory variable has a bounded density, the test error for the learning problem scales as $ \Tilde{\cO}\left(d^{  \frac{2 \lfloor \beta \rfloor( \beta + d)}{2 \beta + d}}n^{-\frac{2\beta}{2\beta + d}}\right)$, showing explicit dependence on the problem dimension ($d$) and the number of samples ($n$) 
    \item Theorem \ref{thm_minmax} establishes that the minimax rates scale as $ \Tilde{\cO}\left(n^{-\frac{2\beta}{2\beta + d}}\right)$, ensuring that deep learners can attain the minimax optimal rate when network sizes are appropriately chosen. Notably, this theorem recovers the seminal results of \cite{yang1999information} as a special case.
    \item When the explanatory variable has an intrinsically low dimensional structure, we show that deep supervised learners can effectively achieve an error rate of $ \Tilde{\cO}\left(n^{-\frac{2\beta}{2\beta + \bar{d}_{2\beta}(\lambda)}}\right)$ in Theorem \ref{main_intrinsic}.  This result provides the fastest known rates for deep supervised learners and encompasses many recent findings as special cases \citep{nakada,chen2022nonparametric} for additive regression models.
    \item In the process, in Lemma \ref{lem_approx} we are able to improve upon the recent $\fL_p$-approximation results on ReLU networks, which might be of independent interest.
\end{itemize}

The remainder of the paper is organized as follows. After discussing the necessary background in Section~\ref{sec_background}, we discuss the exponential family learning framework in Section~\ref{formulation}. Under this framework, we derive the learning rates (Theorem~\ref{main_full}) when the explanatory variable is absolutely continuous in Section~\ref{op_full} and show that it is minimax optimal (Theorem~\ref{thm_minmax}). Next, we analyze the error rate (Theorem~\ref{main_intrinsic}) when the explanatory variable has an intrinsically low dimensional structure in Section~\ref{op_intrinsic}. The proofs of the main results are sketched in Section~\ref{sec_pf}, followed by concluding remarks in Section~\ref{sec_con}.
\section{Background}
\label{sec_background}
This section recalls some of the notation and background that we need.  We say $A_{n,d} \precsim B_{n,d}$ (for $A,\,B \ge 0$) if there exists an absolute constant $C>0$ (independent of $n$ and $d$)%\footnote{PB: are these functions of n and d?}
, such that $A_{n,d} \le C B_{n,d}$, for all $n,d$. Similarly, for non-negative functions $f$ and $g$, we say $f(x) \precsim_x g(x)$ if there exists a constant $C$, which is independent of $x$ such that $f(x) \le C g(x)$, for all $x$.  For any function $f: \cS \to \Real$, and any measure $\gamma$ on $\cS$, let $\|f\|_{\fL_p(\gamma)} : = \left(\int_\cS |f(x)|^p d \gamma(x) \right)^{1/p}$, if $0<p< \infty$. Also let, $\|f\|_{\fL_\infty(\gamma)} : = \operatorname{ess\, sup}_{x \in \text{supp}(\gamma)}|f(x)|$. We say $A_n = \Tilde{\cO}(B_n)$ if $A_n \le B_n \times \log^C(en)$, for some factor constant $C>0$. Moreover, $x \vee y = \max\{x,y\}$ and $x \wedge y = \min\{x,y\}$.
\begin{defn}[Covering and packing numbers] 
    \normalfont 
    For a metric space $(S,\varrho)$, the $\epsilon$-covering number with respect to (w.r.t.) $\varrho$ is defined as:
    \(\cN(\epsilon; S, \varrho) = \inf\{n \in \mathbb{N}: \exists \, x_1, \dots x_n \text{ such that } \cup_{i=1}^nB_\varrho(x_i, \epsilon) \supseteq S\}.\) An $\epsilon$-cover of $S$ is denoted as $\cC(\epsilon;S,\varrho)$.
    Similarly, the $\epsilon$-packing number is defined as:
    \(\cM(\epsilon; S, \varrho) = \sup\{m \in \mathbb{N}: \exists \, x_1, \dots x_m \in S \text{ such that } \varrho(x_i, x_j) \ge \epsilon, \text{ for all } i \neq j\}.\)
\end{defn}
\begin{defn}[Neural networks]\normalfont
 Let $L \in \mathbb{N}$ and $ \{N_i\}_{i \in [L]} \subset \mathbb{N}$. Then a $L$-layer neural network $f: \Real^d \to \Real^{N_L}$ is defined as,
\begin{equation}
\label{ee1}
f(x) = A_L \circ \sigma_{L-1} \circ A_{L-1} \circ \dots \circ \sigma_1 \circ A_1 (x)    
\end{equation}
Here, $A_i(y) = W_i y + b_i$, with $W_i \in \Real^{N_{i} \times N_{i-1}}$ and $b_i \in \Real^{N_{i-1}}$, with $N_0 = d$. Note that $\sigma_j$ is applied component-wise.  Here, $\{W_i\}_{1 \le i \le L}$ are known as weights, and $\{b_i\}_{1 \le i \le L}$ are known as biases. $\{\sigma_i\}_{1 \le i \le L-1}$ are known as the activation functions. Without loss of generality, one can take $\sigma_\ell(0) = 0, \, \forall \, \ell \in [L-1]$. We define the following quantities:  
(Depth) $\cL(f) : = L$ is known as the depth of the network; (Number of weights) The number of weights of the network $f$ is denoted as $\cW(f)$; 
(maximum weight) $\cB(f) = \max_{1 \le j \le L} (\|b_j\|_\infty) \vee \|W_j\|_{\infty}$ to denote the maximum absolute value of the weights and biases.
\begin{align*}
    \cN \cN_{\{\sigma_i\}_{1 \le i \le L-1}} (L, W, R) = \{ & f \text{ of the form \eqref{ee1}}: \cL(f) \le L , \, \cW(f) \le W, \sup_{x \in[0,1]^d}\|f(x)\|_\infty \le R  \}.
\end{align*}
 If $\sigma_j(x) = x \vee 0$, for all $j=1,\dots, L-1$, we denote $\cN \cN_{\{\sigma_i\}_{1 \le i \le L-1}} (L, W, R)$ as $\cR \cN (L, W, R)$. 
 \end{defn}
 \begin{defn}[H\"older functions]\normalfont
Let $f: \mathcal{S} \to \Real$ be a function, where $\mathcal{S} \subseteq \Real^d$. For a multi-index $\bs = (s_1,\dots,s_d)$, let, $\partial^{\bs} f = \frac{\partial^{|\bs|} f}{\partial x_1^{s_1} \dots \partial x_d^{s_d}}$, where, $|\bs| = \sum_{\ell = 1}^d s_\ell $. We say that a function $f: \cS \to \Real$ is $\beta$-H\"{o}lder (for $\beta >0$) if
\[ \|f\|_{\sH^\beta}: =\sum_{\bs: 0 \le |\bs| \le \lfloor \beta \rfloor} \|\partial^{\bs} f\|_\infty  + \sum_{\bs: |\bs| = \lfloor \beta \rfloor} \sup_{x \neq y}\frac{\|\partial^{\bs} f(x)  - \partial^{\bs} f(y)\|}{\|x - y\|^{\beta - \lfloor \beta \rfloor}} < \infty. \]
If $f: \Real^{d_1} \to \Real^{d_2}$, then we define $\|f\|_{\sH^{\beta}} = \sum_{j = 1}^{d_2}\|f_j\|_{\sH^{\beta}}$. For notational simplicity, let, $\sH^\beta(\cS_1, \cS_2,C) = \{f: \cS_1 \to \cS_2: \|f\|_{\sH^\beta} \le C\}$. Here, both $\cS_1$ and $\cS_2$ are both subsets of real vector spaces. 
\end{defn}

\begin{defn}[Smoothness and strong convexity]\normalfont
    We say a differentiable function $f: \Real^d \to \Real$ is $H$-smooth if \(\| \nabla f(x) - \nabla f(y)\|_2 \le H \|x-y\|_2.\) Similarly, we say that $f$ is $\alpha$-strongly convex if \(f(t x + (1-t) y) \le t f(x) + (1-t) f(y) - \frac{\alpha t (1-t) }{2} \|x - y\|_2^2.\)
\end{defn}
\begin{defn}[Bregman divergences]\normalfont
    A differentiable, convex function $\phi: \Real^p \to \Real$ generates the Bregman divergence $d_{\phi}: \Real^p \times \Real^p \to \Real_{\ge 0}$  defined by  \(d_\phi(x\|y) = \phi(x) - \phi(y) - \langle \nabla \phi(y) , x - y \rangle .\)
 \end{defn}
It is evident that $d_\phi(x, y) \geq 0$ holds for all $x, y \in \Real^p$ due to the fact that $\phi(x) \geq \phi(y) + \langle \nabla \phi(y) , x - y \rangle$ is equivalent to the convex nature of the function $\phi$. From a geometric standpoint, one can conceptualize $d_\phi(x\| y)$ as the separation between $\phi(x)$ and the linear approximation of $\phi(x)$ centered around $\phi(y)$. Put simply, this can be described as the distance between $\phi(x)$ and the value obtained by evaluating the tangent line to $\phi(y)$ at the point $x$. Some prominent examples of Bregman divergences include the squared Euclidean distance, Kullback-Leibler (KL) divergence, Mahalanobis distance, etc. We refer the reader to \citet[Table 1]{banerjee2005clustering} for more examples of the Bregman family. Bregman divergences have a direct association with standard exponential families, as elaborated in the upcoming section, rendering them particularly suitable for modeling and learning from various common data types that originate from exponential family distributions.
\section{Learning Framework}\label{formulation}
To discuss our framework, let us first recall the definition of exponential families \citep[Chapter 1, Section 5]{lehmann2006theory}. We suppose that $\btheta \in \Theta$ is the natural parameter. We say that $\bX$ is distributed according to an exponential family, $\sF_{\Psi}$ if the density of $\bX$ w.r.t. some dominating measure $\nu$, is given by, 
\[ p_{\Psi, \theta}(d\bx) = h(\bx) \exp\left( \langle \btheta, T(\bx) \rangle - \Psi(\btheta)\right) \nu(d\bx). \]
Here, $T(\cdot)$ is called the natural statistic. Often, it is assumed that the exponential family is expressed in its regular form, which means that the components of $T(\cdot)$ are affinely independent, i.e. there exists no $\bupsilon$, such that $\langle \bupsilon, T(\bx) \rangle = c$ (a constant), for all $\bx$. Popular examples of exponential families include Gaussian, binomial, Poisson, and exponential distributions. Given an exponential family, one can express it in its natural form. Formally,
\begin{defn}[Natural form of Exponential families]\normalfont
    A multivariate parametric family $\sF_\Psi$ of distributions $\{ p_{\Psi,\btheta} |  \btheta \in \Theta = \operatorname{int}(\Theta) = \operatorname{dom}(\Psi) \subseteq \mathbb{R}^{d_\theta} \}$ is called a regular exponential family provided that each probability density, w.r.t. some dominating measure $\nu$,  is of the form,  \[p_{\Psi,\btheta}(d\bx) = \exp\left(\langle \bx,\btheta\rangle - \Psi(\btheta)\right) h(\bx) \nu(d\bx)\] for all $\bx \in \mathbb{R}^d$, where $\bx$ represents a minimal sufficient statistic for the family.
\end{defn}
It is well known that $\bmu(\btheta) := \E_{\bX \sim p_{\Psi, \btheta}} \bX = \nabla \Psi(\btheta) $. For simplicity, we assume that $\Psi$ is proper, closed, convex, and differentiable. The conjugate of $\Psi$, denoted as $\phi$ is defined as, $\phi(\bmu) = \sup_{\btheta \in \Theta} \left\{\langle \btheta, \bmu \rangle - \Psi(\btheta) \right\}$. It is well known \citep[Theorem 4]{banerjee2005clustering} that $p_{\Psi, \btheta}$ can be expressed as, 
\begin{equation}\label{breg_exp}
    p_{\Psi,\btheta}(d\bx) = \exp(-d_\phi(\bx\| \bmu(\btheta)) b_\phi(\bx) \nu(d\bx),
\end{equation}
where $d_\phi(\cdot \| \cdot)$ denotes the Bergman divergence corresponding to $\phi$. Here, $b_\phi(\bx) = \exp(\phi(\bx)) h(\bx)$. 
In this paper, we are interested in the supervised learning problem when the response, given the explanatory variable, is distributed according to an exponential family. For simplicity, we assume that the responses are real-valued. We assume that there exists a ``true'' predictor function, $f_0: \Real^{d} \to \Real$, such that \[y|\bx \sim p_{\Psi, f_0(\bx)} \quad \text{and} \quad \bx \sim \lambda.\] Thus, the joint distribution of $(\bX, Y)$ is given by,
\begin{equation}
    \label{e1}
    \sP (d\bx, dy) \propto  \exp(-d_\phi(y\| \mu(f_0(\bx))) b_\phi(y) \lambda(d\bx) \nu(dy).
\end{equation}
By definition, we observe that $\E (y|\bx) = \mu(f_0(\bx))$. We will assume that the data is independent and identically distributed according to the distribution $\sP$. We also assume that the distribution of $\bx$ is bounded in the compact set, $[0,1]^d$. Formally,

\begin{ass}\label{model}
    We assume that the data $\{(\bx_i, y_i)\}_{i \in [n]}$ are independent and identically distributed according to the distribution $\sP$, defined in \eqref{e1}. Furthermore, $\lambda \left([0,1]^d\right) = 1$.
\end{ass} 
In the classical statistics literature, one estimates $f_0$ by finding its maximum likelihood estimates (m.l.e.) as, \[\argmax_{f \in \cF} \frac{1}{n}\sum_{i=1}^n \log \sP(\bx_i, y_i).\] 
Plugging in the form of $\sP$ as in \eqref{e1}, it is easy to see that the above optimization problem is equivalent to \begin{equation} \label{mle}
    \argmin_{f \in \cF} \frac{1}{n}\sum_{i=1}^n d_{\phi}\left(y_i \| \mu(f(\bx_i))\right)
\end{equation}

Here, $\mu: \Real \to \Real$ is known as the link function. In practice, we take $\cF$ to be some sort of neural network class, with the final output passing through the activation function $\mu$. The empirical minimizer of \eqref{mle} is denoted as $\hat{f}$. To show that this framework covers a wide range of supervised learning problems, we consider the classical example for the case when $y|\bx \sim \text{Normal}(f_0(\bx),\sigma^2)$, for some unknown $\sigma^2$. In this case, it is well known that $\mu(\cdot)$ is the identity map and $d_\phi(\cdot\| \cdot)$ becomes the squared Euclidean distance. Thus, the m.l.e. problem \eqref{mle} becomes the classical regression problem, i.e. $\argmin_{f \in \cF} \frac{1}{n}\sum_{i=1}^n (y_i- f(\bx_i))^2$.

Another example is the case of logistic regression. We assume that $y|\bx$ is a Bernoulli random variable. This makes $\mu(z) = \frac{1}{1+e^{-z}}$, i.e. the sigmoid activation function. Furthermore, an easy calculation \citep[Table 2]{banerjee2005clustering} shows that $d_\phi(x \| y) = x \log\left(\frac{x}{y}\right) + (1-x) \log\left(\frac{1-x}{1-y}\right)$. Plugging in the values of $\mu$ and $d_\phi$ into \eqref{mle}, we note that the estimation problem becomes, 
\[\argmin_{f \in \cF} -\frac{1}{n}\sum_{i=1}^n \left(y_i \log \circ \, \text{sigmoid}(f(\bx_i)) + (1- y_i) \log \circ (1-\text{sigmoid}(f(\bx_i)))\right).\]
Here, $\operatorname{sigmoid}(t) = 1/(1+e^{-t})$ denotes the sigmoid activation function. Thus, the problem reduces to the classical two-class learning problem with the binary cross-entropy loss and with sigmoid activation at the final layer. For simplicity, we assume that all activations, excluding that of the final layer of $f$, are realized by the ReLU activation function. The choice of ReLU activation is a natural choice for practitioners and enables us to harness the approximation theory of ReLU networks developed throughout the recent literature \citep{yarotsky2017error, uppal2019nonparametric}. However, using a leaky ReLU network will also result in a similar analysis, changing only the constants in the main theorems.

To facilitate the theoretical analysis, we will assume that the problem is smooth in terms of the learning function $f_0$. As a notion of smoothness, we will use H\"{o}lder smoothness. This has been a popular choice in the recent literature \citep{nakada,schmidt2020nonparametric,chen2022nonparametric} and covers a vast array of functions commonly modeled in the literature.
\begin{ass}\label{a2}
    $f_0$ is $\beta$-H\"{o}lder continuous, i.e. $f_0 \in \sH^{\beta}(\Real^d, \Real,C)$.
\end{ass}
We make the additional assumption that the function $\Psi$ is well-behaved. In particular, we assume that $\Psi$ possesses both smoothness and strong convexity properties. It is important to note that these assumptions are widely employed in the existing literature \citep{telgarsky2013moment,paul2021uniform}. Though A\ref{a4} is not applicable for the classification problem, as in that case, $\Psi(x) = x \ln x + (1-x) \ln (1-x)$, which does not satisfy A\ref{a4}. However for all practical purposes, one clips the output network (which is often done in practice to ensure smooth training), i.e. ensures that $ \epsilon \le f,f_0 \le 1-\epsilon$, for some positive $\epsilon$, A\ref{a4} is satisfied. The assumption is formally stated as follows.
\begin{ass}\label{a4}
    We assume that $\Psi$ is $\sigma_1$-smooth and $\sigma_2$-strongly convex.
\end{ass}
A direct implication of A\ref{a4} is that by \citet[Theorem 6]{kakade2009duality}, $\phi$ is $\tau_2$-smooth and $\tau_1$- strongly convex. Here $\tau_i = 1/\sigma_i$. Also, since $\Psi$ is $\sigma_1$-smooth, $\mu(\cdot) = \nabla \Psi(\cdot) $ is $\sigma_1$-Lipschitz. This fact will be useful for the proofs of the main results. In the subsequent sections, under the above assumptions, we derive probabilistic error bounds for the excess risk of $\hat{f}$.

\section{Optimal Rates for Distributions with Bounded Densities}\label{op_full}
We begin the analysis of the test error for the problem \eqref{mle} when $\lambda$, the distribution of the explanatory variable has a bounded density on $[0,1]^d$. First note that the excess risk is upper bounded by the estimation error for $f_0$ in the $\fL_2(\lambda)$-norm. The excess risk for the problem is given by 
\[ \mathfrak{R}(\hat{f}) = \E_{(y,\bx) \sim \sP} d_\phi (y \| \mu(\hat{f}(\bx))) - \E_{(y,\bx) \sim \sP} d_\phi (y \| \mu(f_0(\bx))). \]
The following lemma ensures that $\mathfrak{R}(\hat{f}) \asymp  \| \hat{f} - f_0  \|_{\fL_2(\lambda)}^2$ and hence, it is enough to prove bounds on $\|\hat{f}- f_0 \|_{\fL_2(\lambda)}^2$ to derive upper and lower bounds on the excess risk.
\begin{restatable}{lem}{lemoct}\label{lemoct}
    For any $\hat{f} \in \cF$, 
    \( \frac{\sigma_2}{\sigma_1} \| \hat{f} - f_0 \|_{\fL_2(\lambda)}^2 \le \mathfrak{R}(\hat{f}) \le \frac{\sigma_1}{\sigma_2} \| \hat{f} - f_0 \|_{\fL_2(\lambda)}^2.\)
\end{restatable}
 
As already mentioned, we assume that $\lambda$ admits a density w.r.t. the Lebesgue measure, and this density is upper bounded. Formally,
\begin{ass}\label{a5}
    Suppose that $\lambda$ admits an upper-bounded density $p_\lambda$ w.r.t. the Lebesgue measure on $[0,1]^d$, i.e. $\|p_\lambda\|_\infty \le \bar{b}_\lambda$, almost surely (under the Lebesgue measure).  
\end{ass}
Under Assumptions~A\ref{model}--\ref{a5}, we observe that with high probability, $\|\hat{f} - f_0\|^2_{\fL_2(\lambda)}$ and thereby, $\mathfrak{R}(\hat{f})$ scales at most as $\Tilde{\cO}\left(d^{2 \lfloor \beta \rfloor} n^{-\frac{2\beta}{2\beta + d}} \right)$, ignoring $\log$-terms in $n$, for large $n$, if the network sizes are chosen properly. This rate of decrease aligns consistently with prior findings reported by \cite{nakada} for additive Gaussian-noise regression. It is important to underscore that the existing literature predominantly investigates the rate of decrease in $\mathfrak{R}(\hat{f})$ solely with regard to the sample size $n$, overlooking terms dependent upon the data dimensionality. These dimension-dependent terms harbor the potential for exponential growth with respect to the dimensionality of the explanatory variables, and may therefore attain substantial magnitudes, making such bounds inefficient in high-dimensional statistical learning contexts. This analysis shows that the dependence in $d$ is not exponential and can be made to increase at a polynomial rate only under the assumption of the existence of a bounded density of the explanatory variable. The main upper bound for this case is stated in Theorem~\ref{main_full}, with a proof outline appearing in Section~\ref{pf_upper}.

\begin{thm}\label{main_full}
    Suppose Assumptions A\ref{model}--\ref{a5} hold. Then we can choose $\cF = \cR\cN(L,W,2C)$ with $L \precsim \log n$ and $W \precsim n^{\frac{d}{2 \beta + d}} \log n$, such that, with probability at least $1 - 3\exp\left(-n^{\frac{d}{2 \beta + d }}\right) $,
\[ \|\hat{f} - f_0\|^2_{\fL_2(\lambda)} \precsim d^{  \frac{2 \lfloor \beta \rfloor( \beta + d)}{2 \beta + d}}n^{-\frac{2\beta}{2\beta + d}} (\log n)^5,\]
for $n \ge n_0$, where, $n_0$ might depend on $d$.
\end{thm}

From the bound on the network size, i.e. $W$ in Theorem~\ref{main_full}, it is clear that when $f_0$ is smooth i.e. for large $\beta$, one requires a network of smaller size compared to when $f_0$ is less smooth i.e. when $\beta$ is small. Similarly, in cases where the dimensionality of the explanatory variables, represented by $d$ is substantial, a larger network is required as compared to situations where $d$ is relatively small. This observation aligns with the intuitive expectation that solving more intricate and complex problems in higher dimensions demands the utilization of larger networks.

The high probability bound in Theorem \ref{main_full} ensures that the expected test error also scales with the same rate of convergence. This result is shown in Corollary \ref{cor_1}
\begin{cor}\label{cor_1}
    Under the assumptions and choices of Theorem \ref{main_full}, $\E \|\hat{f} - f_0\|^2_{\fL_2(\lambda)} \precsim d^{  \frac{2 \lfloor \beta \rfloor( \beta + d)}{2 \beta + d}} n^{-\frac{2\beta}{2\beta + d}} (\log n)^5 $.
\end{cor}
To understand whether deep supervised learning can achieve the optimal rates for the learning problem, we derive the minimax rates for estimating $f_0$. The minimax expected risk for this problem, when one has access to $n$ i.i.d. samples $\{(\bx_i, y_i)\}_{i\in [n]}$ from \eqref{e1} ($f_0$ replaced with $f$) is given by
\[\mathfrak{M}_n = \inf_{\hat{f}} \sup_{f \in \sH^\beta(\Real^d, \Real,C)} \E_{f} \|\hat{f} - f\|_{\fL_2(\lambda)}^2,\]
With the notation $\E_f (\cdot)$ we denote the expectation w.r.t. the measure \eqref{e1}, with $f_0$ replaced with $f$. Here the infimum is taken over all measurable estimates of $\hat{f}$, based on the data. Minimax lower bounds are used to understand the theoretical limits of any statistical estimation problem. The aim of this analysis is to show that deep learning with ReLU networks for the exponential family dependence is (almost) minimax optimal. To facilitate the theoretical analysis, we assume that the density of $\lambda$ is lower bounded by a positive constant. Formally,
\begin{ass}\label{a6}
    $\lambda$ admits a lower-bounded density $p_\lambda$ w.r.t. the Lebesgue measure on $[0,1]^d$, i.e.  $p_\lambda(\bx) \ge \underline{b}_\lambda$, almost surely (under the Lebesgue measure).  
\end{ass}
 Theorem~\ref{thm_minmax} provides a characterization of this minimax lower bound for estimating $f$. It is important to note that the seminal works of \cite{yang1999information} for the normal-noise regression problem is a special case of Theorem~\ref{thm_minmax}.
\begin{thm}[Minimax lower bound]\label{thm_minmax}
    Suppose that Assumptions A\ref{model}--\ref{a4} and A\ref{a6} hold. Then, we can find an $n_0 \in \mathbb{N}$, such that, if $n \ge n_0$,
    \[\inf_{\hat{f}} \sup_{f \in \sH^\beta(\Real^d, \Real,C)} \E_{f} \|\hat{f} - f\|_{\fL_2(\lambda)}^2 \succsim_n  n^{-\frac{2\beta}{2\beta + d}}.\]
\end{thm}
Thus, from Theorems~\ref{main_full} and~\ref{thm_minmax} it is clear that deep supervised estimators for the exponential family dependence can achieve this minimax optimal rate with high probability, barring an excess poly-log factor of $n$.

\section{Rates for Low Intrinsic Dimension}\label{op_intrinsic}

Frequently, it is posited that real-world data, such as vision data, resides within a lower-dimensional structure embedded in a high-dimensional feature space \citep{pope2020intrinsic}. To quantify this intrinsic dimensionality of the data, researchers have introduced various measures of the effective dimension of the underlying probability distribution assumed to generate the data. Among these approaches, the most commonly used ones involve assessing the rate of growth of the covering number, in a logarithmic scale, for most of the support of this data distribution. Let us consider a compact Polish space denoted as $(\cS, \varrho)$, with $\gamma$ representing a probability measure defined on it. For the remainder of this paper, we will assume that $\varrho$ corresponds to the $\ell_\infty$-norm. The simplest measure of the dimension of a probability distribution is the upper Minkowski dimension of its support, defined as follows:
\[\overline{\text{dim}}_M(\gamma) = \limsup_{\epsilon \downarrow 0} \frac{\log\cN(\epsilon;\text{supp}(\gamma), \ell_\infty)}{\log (1/ \epsilon)}.\] 
This dimensionality concept relies solely on the covering number of the support and does not assume the existence of a smooth mapping to a lower-dimensional Euclidean space. Consequently, it encompasses not only smooth Riemannian manifolds but also encompasses highly non-smooth sets like fractals. The statistical convergence properties of various estimators concerning the upper Minkowski dimension have been extensively explored in the literature. \citet{kolmogorov1961} conducted a comprehensive study on how the covering number of different function classes depends on the upper Minkowski dimension of the support. Recently, \citet{nakada} demonstrated how deep learning models can leverage this inherent low-dimensionality in data, which is also reflected in their convergence rates. Nevertheless, a notable limitation associated with utilizing the upper Minkowski dimension is that when a probability measure covers the entire sample space but is concentrated predominantly in specific regions, it may yield a high dimensionality estimate, which might not accurately reflect the underlying dimension.

To overcome the aforementioned difficulty, as a notion of the intrinsic dimension of a measure $\gamma$, \citet{chakraborty2024statistical} introduced the notion of $\alpha$-entropic dimension of a measure. Before we proceed, we recall the $(\epsilon, \tau)$-cover of a measure 
\citep{posner1967epsilon} as: \(\sN_\epsilon(\gamma, \tau) = \inf\{\cN(\epsilon; S, \varrho): \gamma(S) \ge 1-\tau\},\)  i.e. $\sN_\epsilon(\gamma, \tau)$ counts the minimum number of $\epsilon$-balls required to cover a set $S$ of probability at least $1-\tau$.
 \begin{defn}[Entropic dimension, \citealp{chakraborty2024statistical}]\label{ed}\normalfont
    For any $\alpha>0$, we define the $\alpha$-entropic dimension of $\gamma$ as:
     \[\bar{d}_\alpha(\gamma) = \limsup_{\epsilon \downarrow 0} \frac{\log \sN_\epsilon(\gamma,\epsilon^\alpha)}{\log (1/\epsilon)}.\]
\end{defn}

 This notion extends Dudley's notion of entropic dimension \citep{dudley1969speed} to characterize the convergence rate for the BiGAN problem \citep{donahue2017adversarial}. \citet{chakraborty2024statistical} showed that the entropic dimension is not larger than the upper Minkowski dimension and the upper Wasserstein dimension \citep{weed2019sharp}. Furthermore, strict inequality holds even for simplistic examples for measures on the unit hypercube. We refer the reader to Section~3 of \citet{chakraborty2024statistical}. \citet{chakraborty2024statistical} showed that the entropic dimension is a more efficient way of characterizing the intrinsic dimension of the data distributions compared to popular measures such as the upper Minkowski dimension or the Wasserstein dimension \citep{weed2019sharp} as the entropic dimension enables us to derive a faster rate of convergence of the estimates. Importantly, $\bar{d}_\alpha(\gamma) \le \overline{dim}_M(\gamma)$, with strict inequality holding, even for simplistic cases (see examples 10 and 11 of \citealp{chakraborty2024statistical}). 

As an intrinsically low-dimensional probability measure is not guaranteed to be dominated by the Lebesgue measure, we remove Assumption A\ref{a5}. Under only Assumptions A\ref{model}--\ref{a4}, if the network sizes are properly chosen, the rate of convergence of $\hat{f}$ to $f_0$ under the $\fL_2(\lambda)$-norm decays at a rough rate of $\Tilde{\cO}\left(n^{-2\beta /(2\beta + \bar{d}_{2\beta}(\lambda))}\right)$, as shown by Theorem~\ref{main_intrinsic}.
\begin{thm}\label{main_intrinsic}
    Suppose Assumptions A\ref{model}--\ref{a4} holds and let $d^\star > \bar{d}_{2\beta}(\lambda)$. Then we can choose $\cF = \cR\cN(L,W,2C)$ with $L \precsim \log n$ and $W \precsim n^{\frac{d^\star}{2 \beta + d^\star}} \log n$, such that, with probability at least $1 - 3\exp\left(-n^{\frac{d^\star}{2 \beta + d^\star}}\right) $,
\[ \|\hat{f} - f_0\|^2_{\fL_2(\lambda)} \precsim_n  n^{-\frac{2\beta}{2\beta + d^\star}} (\log n)^5,\]
for $n \ge n_0$, where, $n_0$ depends on $d$ and $\sP$.
\end{thm}

Since the normal-noise regression model with $\beta$-H\"{o}lder $f_0$ is a special case of our model in \eqref{e1}, Theorem~\ref{main_intrinsic} derives a faster rate compared to \citet[Theorem 7]{nakada}, who show a rate of $\Tilde{\cO}\left(n^{-\frac{2\beta}{2\beta + \overline{\text{dim}}_M(\lambda)}}\right)$. This is because the upper Minkowski dimension is at least the $2\beta$-entropic dimension by \citet[Proposition 8(c)]{chakraborty2024statistical}, i.e. $\bar{d}_{2\beta}(\lambda) \le \overline{\text{dim}}_M(\lambda) $. 

An immediate corollary of Theorem~\ref{main_intrinsic} is that the expected test-error rate follows the same rate of decay. The proof of this result can be done following the proof of Corollary \ref{cor_1}.
\begin{cor}
    Under the assumptions and choices of Theorem \ref{main_intrinsic}, $\E \|\hat{f} - f_0\|^2_{\fL_2(\lambda)} \precsim  n^{-\frac{2\beta}{2\beta + d^\star}} (\log n)^5 $.
\end{cor}
One can state that a rate similar to that observed in Theorem~\ref{main_intrinsic} holds when the support of $\lambda$ is regular. We recall the definition \citep[Definition 6]{weed2019sharp} of regular sets in $[0,1]^d$ as follows.
\begin{defn}[Regular sets]
    We say  a set $\fM$ is $\tilde{d}$-regular w.r.t. the $\tilde{d}$-dimensional Hausdorff measure $\fH^{\tilde{d}}$, if \(\fH^{\tilde{d}}(B_\varrho(x, r)) \asymp r^{\tilde{d}},\)
for all $x \in \fM$.  Recall that the $d$-Hausdroff measure of a set $S$ is defined as, \(\fH^d(S) := \liminf_{\epsilon \downarrow 0} \left\{\sum_{k=1}^\infty r_k^d : S \subseteq \sum_{k=1}^\infty B_\varrho (x_k, r_k), r_k \le \epsilon, \, \forall k\right\}.\)
\end{defn}
Examples of regular sets include compact $\tilde{d}$-dimensional differentiable manifolds; nonempty, compact convex sets spanned by an affine space of dimension $\tilde{d}$; the relative boundary of a nonempty, compact convex set of dimension $\tilde{d}+1$; or a self-similar set with similarity dimension $\tilde{d}$. When the support of $\lambda$ is $\tilde{d}$-regular, it can be shown that $\bar{d}_\alpha(\lambda) \le \tilde{d}$. Formally,
\begin{restatable}{lem}{lemtwelve}\label{lem12}
     Suppose that the support of $\gamma$ is $\tilde{d}$-regular. Then, $\bar{d}_\alpha(\gamma) \le \tilde{d}$, for any $\alpha>0$. Furthermore, if $\gamma \ll \fH^{\tilde{d}}$, $ \bar{d}_\alpha(\gamma) = \tilde{d}$. 
\end{restatable}
Thus, applying Theorem~\ref{main_intrinsic} and Lemma~\ref{lem12}, we note that $\|\hat{f}-f_0\|_{\mathbb{L}_2(\lambda)}^2$ decays at most at a rate of $\Tilde{\cO}\left(n^{-2\beta /(2\beta + \tilde{d})}\right)$, resulting in the following corollary.
\begin{cor}\label{cor10}
    Suppose Assumptions A\ref{model}--\ref{a4} and the support of $\lambda$ is $\tilde{d}$-regular.  Let $d^\star > \tilde{d}$, then we can choose $\cF = \cR\cN(L,W,2C)$ with $L \precsim \log n$ and $W \precsim n^{\frac{d^\star}{2 \beta + d^\star}} \log n$, such that, with probability at least $1 - 3\exp\left(-n^{\frac{d^\star}{2 \beta + d^\star}}\right) $,
\( \|\hat{f} - f_0\|^2_{\fL_2(\lambda)} \precsim_n  n^{-\frac{2\beta}{2\beta + d^\star}} (\log n)^5,\)
for $n \ge n_0$, where, $n_0$ depends on $d$ and $\sP$.
\end{cor}
    Since compact $\tilde{d}$-dimensional differentiable manifolds are a special case of $\tilde{d}$-regular sets, Corollary~\ref{cor10} recovers the results by \cite{chen2022nonparametric} as a special case i.e.,  an additive Gaussian-noise regression model. Importantly, this recovery is achieved without imposing assumptions about uniform sharpness on the manifold, as done by \citet[Assumption 2]{chen2022nonparametric}.
 
\section{Proof of the Main Results}\label{sec_pf}
This section discusses the proof of the main results of this paper, i.e, Theorems \ref{main_full}, \ref{thm_minmax} and \ref{main_intrinsic}, with proofs of auxiliary supporting lemmas appearing in the appendix. The proof of the main upper bounds (Theorems \ref{main_full} and \ref{main_intrinsic}) are presented in Section \ref{pf_upper}, while the minimax lower bound is proved in Section \ref{pf_lower}.
\subsection{Proof of the Upper Bounds}\label{pf_upper}

In order to prove Theorem~\ref{main_full}, we first decompose the error through an oracle inequality. For any vector $v \in \Real^q$, we denote $\vertiii{v }_{p,q} = \left(\frac{1}{q} \sum_{i=1}^q |v_i|^p\right)^{1/p}$.

\begin{restatable}[Oracle inequality]{lem}{lemoracle}\label{lem_oracle}
Let $f^\ast \in \cF$.  Suppose that $\xi_i  = y_i - \mu(f_0(\bx_i))$, $\hat{\Delta}_i = \mu(\hat{f}(\bx_i)) - \mu(f_0(\bx_i))$ and $\Tilde{\Delta}_i = \nabla \phi(\mu(f^\ast(\bx_i))) - \nabla \phi(\mu(f_0(\bx_i)))$. Then, 
\begin{equation}\label{e_11}
    \tau_1 \vertiii{\hat{\Delta}}_{2,n}^2 \le \tau_2 \|\mu(f^\ast) - \mu(f_0) \|_{\fL_2(\lambda_n)}^2 + \frac{1}{n} \sum_{i=1}^n \xi_i \Tilde{\Delta}_i.
\end{equation}
\end{restatable}

The first term in the right-hand side (RHS) of \eqref{e_11} is analogous to an approximation error while the second term is akin to a generalization gap. It is worth noting that while taking a large network reduces the approximation error, it can potentially give rise to a large generalization gap and vice versa. The key idea is to select a network of appropriate size that ensures that both these errors are small enough. In the following two sections, we control these terms individually.
\subsubsection{Generalization Error}\label{gen_sec}
To effectively control the generalization error, we employ standard localization techniques; see, for example,~\citet[Chapter 14]{wainwright_2019}. These techniques are instrumental in achieving rapid convergence of the sample estimator to the population estimator in the $\fL_2(\lambda)$ norm. It is important to note that in some cases, the true function, denoted as $f_0$, may not be precisely representable by a ReLU network. We establish a high-probability bound for the squared $\fL_2(\lambda)$ norm difference between our estimated function $\hat{f}$ and $f^\ast$, where we will take $f^\ast$ to belong in the neural network function class, close enough to $f_0$. Our strategy revolves around a two-step process: firstly, we derive a local complexity bound, as outlined in Lemma~\ref{lem_17.3} and subsequently, we leverage this local complexity bound to derive an estimate for $\|\hat{f} - f^\ast\|^2_{\fL_2(\lambda_n)}$, as elucidated in Lemma~\ref{lem_17.2}. Here $\lambda_n$ denotes the empirical distribution of the explanatory variables. We then use this result to control $\|\hat{f} - f^\ast\|^2_{\fL_2(\lambda)}$ in Lemma~\ref{lem_17.4} for large $n$. We state these results subsequently with proofs appearing in Appendix~\ref{ap1}.

\begin{restatable}{lem}{lemel}
    \label{lem_17.3}
    Suppose that 
\(\sG_\delta = \left\{\nabla\phi(\mu(f))- \nabla \phi(\mu(f^\prime)): \|f-f^\prime\|_{\fL_\infty(\lambda_n)} \le \delta \text{ and } f,f^\prime \in \cF \right\}\), with $\delta \le 1/e$. Also let, $n \ge \operatorname{Pdim}(\cF)$. Then, for any $t>0$, with probability (conditioned on $x_{1:n}$) at least $1 -  e^{-nt^2/\delta^2}$,
\begin{align}
     \sup_{g \in \sG_\delta} \frac{1}{n}\sum_{i=1}^n \xi_i g(x_i) \precsim & \, t + \delta \sqrt{\frac{\operatorname{Pdim}(\cF) \log (n/\delta)}{n}}. 
     \label{e11}
\end{align}
Here, $\operatorname{Pdim}(\cF)$ denotes the pseudo-dimension of the function class $\cF$ \citep{anthony1999neural}.
\end{restatable}

\begin{restatable}{lem}{lemgentwo}
    \label{lem_17.2}
    Suppose $\alpha \in (0,1/2)$ and $n \ge \max\left\{e^{1/\alpha}, \operatorname{Pdim}(\cF)\right\}$. Then, for any $f^\ast \in \cF$, with probability at least, $1 -  \exp\left(-n^{1-2 \alpha}\right)$, 
\begin{equation}
    \|\hat{f} - f^\ast\|^2_{\fL_2(\lambda_n)} \precsim  n^{-2 \alpha} +\|f^\ast - f_0\|^2_{\fL_2(\lambda_n)} +  \frac{1}{n }\operatorname{Pdim}(\cF) \log n\label{ee_s5}
\end{equation}
\end{restatable}

\begin{restatable}{lem}{lemgenone}\label{thm_17.1}
    For $\alpha \in (0,1/2)$, if $n \ge \max\left\{e^{1/\alpha}, \operatorname{Pdim}(\cF)\right\}$, with probability at least $1 - 3 \exp\left(-n^{1-2\alpha}\right)$
\begin{align}
    \|\hat{f} - f_0\|^2_{\fL_2(\lambda)} \precsim n^{-2 \alpha} +\|f^\ast - f_0\|^2_{\fL_2(\lambda)}+  \frac{1}{n }\operatorname{Pdim}(\cF) \log^2 n + \frac{1}{n}\log \log n,
\end{align}
for any $f^\ast \in \cF$.
\end{restatable}
In Lemmata \ref{lem_17.2} and \ref{thm_17.1}, one can think of $f^\ast$ as the closest member of $\cF$ to $f_0$, making the term, $\|f^\ast - f_0\|^2_{\fL_2(\lambda)}$ akin to a misspecification error. The intuition is to choose $\cF$ appropriately, so that the misspecification and the generalization errors in Lemma \ref{thm_17.1} are both small.

\subsubsection{Approximation Error}\label{ap_sec}
To effectively bound the overall error in Lemma~\ref{lem_oracle}, one needs to control the approximation error, denoted by the first term of \eqref{e_11}. Exploring the approximating potential of neural networks has witnessed substantial interest in the research community in the past decade or so. Pioneering studies such as those by \cite{cybenko1989approximation} and \cite{hornik1991approximation} have extensively examined the universal approximation properties of networks utilizing sigmoid-like activations. These foundational works demonstrated that wide, single-hidden-layer neural networks possess the capacity to approximate any continuous function within a bounded domain. In light of recent advancements in deep learning, there has been a notable surge in research dedicated to exploring the approximation capabilities of deep neural networks. Some important results in this direction include those by \cite{yarotsky2017error,lu2021deep,petersen2018optimal,shen2019nonlinear,schmidt2020nonparametric} among many others. All of the aforementioned results indicate that when $\epsilon$-approximating a $\beta$-H\"{o}lder function in the $\ell_\infty$-norm, it suffices to have a network of depth $\cO(\log(1/\epsilon))$  with at most $\cO( \epsilon^{-d/\beta} \log(1/\epsilon))$-many weights for the approximating network. However, the constants in the expressions of the upper bound of the number of weights and depth of the network can potentially increase exponentially with $d$. \cite{shen2022optimal} showed that if one approximates in the $\fL_2(\text{Leb})$-norm, this exponential dependence can be mitigated for the case $\beta \le 1$. Here $\text{Leb}(\cdot)$ denotes the Lebesgue measure on $[0,1]^d$. Lemma~\ref{lem_approx} generalizes this result to include all $\beta >0$ to achieve a precise dependence on $d$. The proof is provided in Appendix \ref{pf_approx}.
\begin{restatable}{lem}{lemapprox}
\label{lem_approx}
    Suppose that $f \in \sH^\beta(\Real, \Real, C)$. Then, we can find a ReLU network, $\hat{f}$, with $\cL(\hat{f}) \le \vartheta \lceil \log_2(8/\eta) \rceil + 4$ and $\cW(\hat{f}) \le \left\lceil \frac{1}{2(\eta/20)^{1/\beta}}\right\rceil^d \left(\frac{3}{\beta}\right)^\beta (d+\lfloor\beta \rfloor)^{\lfloor \beta \rfloor} \left(\vartheta \left\lceil \log_2 \left(\frac{8}{\eta d^{\lfloor \beta \rfloor}}\right) \right\rceil + 8 d + 4 \lfloor \beta \rfloor \right)$, and a constant $\eta_0 \in (0,1)$ (that might depend on $\beta$ and $d$) such that $\|f - \hat{f}\|_{\fL_p(\text{Leb})} \le C d^{\lfloor \beta \rfloor} \eta,$ for all $\eta \in (0, \eta_0]$. Here, $\vartheta$ is an absolute constant. 
\end{restatable}
We now provide formal proofs of Theorems~\ref{main_full} and \ref{main_intrinsic} by combining the results in Sections \ref{gen_sec} and \ref{ap_sec}.
\subsubsection{Proof of Theorem~\ref{main_full}}
\begin{proof}
We take $\cF = \cR\cN(L_\epsilon, W_\epsilon, 2 C)$, with $L_\epsilon \precsim \log(1/\epsilon)$ and $W_\epsilon \precsim d^{\lfloor \beta \rfloor} \epsilon^{-d/\beta} \log(1/\epsilon)$. Then, by Lemma~\ref{lem_approx}, we can find $f^\ast \in \cF$, such that, $\|f^\ast - f_0\|_{\fL_2(\lambda)} \precsim d^{\lfloor \beta \rfloor} \epsilon$. Furthermore, by Lemma~\ref{thm_17.1}, we observe that with probability at least $1-3 \exp\left(-n^{1-2\alpha}\right)$, 
\begin{align}
    \|\hat{f} - f_0\|^2_{\fL_2(\lambda)} \precsim & n^{-2 \alpha} +\|f^\ast - f_0\|^2_{\fL_2(\lambda)}+  \frac{1}{n }\operatorname{Pdim}(\cF) \log^2 n + \frac{\log \log n}{n} \nonumber\\
    \le & n^{-2 \alpha} + d^{2 \lfloor \beta \rfloor}\epsilon^2 +  \frac{1}{n }\operatorname{Pdim}(\cF) \log^2 n + \frac{\log \log n}{n} \nonumber\\
    \precsim & n^{-2 \alpha} + d^{2 \lfloor \beta \rfloor} \epsilon^2 +  \frac{\log^2 n }{n } W_\epsilon L_\epsilon \log(W_\epsilon)  + \frac{\log \log n}{n} \nonumber\\
    \precsim & n^{-2 \alpha} + d^{2 \lfloor \beta \rfloor} \epsilon^2 +  \frac{d^{\lfloor \beta \rfloor}\log^2 n}{n } \epsilon^{-d/\beta } \log^3(1/\epsilon) + \frac{\log \log n}{n}. \label{e_18.1}
\end{align}
Here \eqref{e_18.1} follows from the following calculations. Suppose $\alpha_2$ is the constant that honors $W_\epsilon \precsim d^{\lfloor \beta \rfloor} \epsilon^{-d/\beta} \log(1/\epsilon)$, i.e. $W_\epsilon \le \alpha_2 d^{\lfloor \beta \rfloor} \epsilon^{-d/\beta} \log(1/\epsilon)$. Then,
\[\log W_\epsilon \le \log \alpha_2  + \lfloor \beta \rfloor \log d  + \frac{d}{\beta}\log(1/\epsilon)+ \log \log(1/\epsilon) \le \frac{3d}{\beta}\log(1/\epsilon),\]
when $\epsilon$ is small enough. Taking $\epsilon \asymp (n d^{\lfloor\beta \rfloor})^{-\frac{\beta}{2 \beta + d}}$ and $\alpha = \frac{\beta}{2 \beta + d }$, we note that with probability at least $1 - 3\exp\left(-n^{\frac{d}{2 \beta + d}}\right) $,
\[ \|\hat{f} - f_0\|^2_{\fL_2(\lambda)} \precsim d^{  \frac{2 \lfloor \beta \rfloor( \beta + d)}{2 \beta + d}}n^{-\frac{2\beta}{2\beta + d}} (\log n)^5.\]
Note that for the above bounds to hold, one requires $n \ge \operatorname{Pdim}(\cF)$ and $\epsilon \le \epsilon_0$, which holds when $n$ is large enough.
\end{proof}

\subsubsection{Proof of Theorem~\ref{main_intrinsic}}
\begin{proof}
We take $\cF = \cR\cN(L_\epsilon, W_\epsilon, 2 C)$, with $L_\epsilon \precsim_\epsilon \log(1/\epsilon)$ and $W_\epsilon \precsim_\epsilon \epsilon^{-d^\star/\beta} \log(1/\epsilon)$. Then, by \citet[Theorem 18]{chakraborty2024statistical}, %Theorem~\ref{approx},
we can find $f^\ast \in \cF$, such that, $\|f^\ast - f_0\|_{\fL_2(\lambda)} \le \epsilon$. Furthermore, by Lemma~\ref{thm_17.1}, we observe that with probability at least $1-3 \exp\left(-n^{1-2\alpha}\right)$, 
\begin{align}
    \|\hat{f} - f_0\|^2_{\fL_2(\lambda)} \precsim & n^{-2 \alpha} +\|f^\ast - f_0\|^2_{\fL_2(\lambda)}+  \frac{1}{n }\operatorname{Pdim}(\cF) \log^2 n + \frac{\log \log n}{n} \nonumber\\
    \le & n^{-2 \alpha} +\epsilon^2 +  \frac{1}{n }\operatorname{Pdim}(\cF) \log^2 n + \frac{\log \log n}{n} \nonumber\\
    \precsim & n^{-2 \alpha} +\epsilon^2 +  \frac{1}{n } W_\epsilon L_\epsilon \log(W_\epsilon) \log^2 n + \frac{\log \log n}{n} \label{e_17.1}\\
    \precsim_\epsilon & n^{-2 \alpha} +\epsilon^2 +  \frac{\log^2 n}{n } \epsilon^{-d^\star/\beta } \log^3(1/\epsilon) + \frac{\log \log n}{n} \nonumber.
\end{align}
Here, \eqref{e_17.1} follows from \cite[Theorem~6]{bartlett2019nearly}.  %Lemma~\ref{bartlett_vc}. 
Taking $\epsilon \asymp n^{-\frac{\beta}{2 \beta + d^\star}}$ and $\alpha = \frac{\beta}{2 \beta + d^\star}$, we note that, with probability at least $1 - 3\exp\left(-n^{\frac{d^\star}{2 \beta + d^\star}}\right) $,
\( \|\hat{f} - f_0\|^2_{\fL_2(\lambda)} \precsim_n n^{-\frac{2\beta}{2\beta + d^\star}} (\log n)^5.\)
Note that for the above bounds to hold, one requires $n \ge \operatorname{Pdim}(\cF)$ and $\epsilon \le \epsilon_0$, which holds when $n$ is large enough.
\end{proof}
\subsection{Proof of the Minimax Rates}\label{pf_lower}
In this section, we give a formal proof of Theorem~\ref{thm_minmax}.  We use the standard technique of Fano's method---see, for example,~\citep[Chapter 15]{wainwright_2019}---to construct hypotheses that are well separated in $\fL_2(\lambda)$ sense but difficult to distinguish in the KL-divergence.

\textbf{Proof of Theorem~\ref{thm_minmax}:}
%\begin{proof}
 Let $b(x) = \exp\left(\frac{1}{x^2 -1} \right) \mathbbm{1}\{|x| \le 1\}$ be the standard bump function on $\Real$. For any $x \in \Real^d$ and $\delta \in (0,1]$, we let, $h_\delta(x) = a \delta^\beta \prod_{j=1}^d b(x_j/\delta)$. Here $a$ is such that $a b(x) \in \sH^\beta(\Real, \Real, C)$. It is easy to observe that $h_\delta \in \sH^\beta(\Real^d, \Real, C)$. In what follows, we take, $\delta = 1/m$. Let
 \[\sF_\delta = \left\{f_{\bomega}(x) = \sum_{\bxi \in [m]^d} \omega_{\bxi} h_\delta\left(x-\frac{1}{m}(\xi_i-1/2)\right): \bomega \in\{0,1\}^{m^d}\right\}.\]

 Since each element of $\sF_\delta$ is a sum of members in $\sH^\beta(\Real, \Real, C)$ with disjoint support, $\sF_\delta \subseteq \sH^\beta(\Real, \Real, C)$. By the Varshamov-Gilbert bound \citep[Lemma~2.9]{tsybakov2009introduction}, %(Lemma~\ref{vg bound}), 
 we can construct a subset of $ \Omega = \{\bomega_1, \dots, \bomega_M\}$ of $\{0,1\}^{m^d}$ with $\|\bomega_i - \bomega_j\|_1 \ge \frac{m^d}{8}$, for all $i \neq j$ and $M \ge 2^{m^d/8}$. We note that for any $\bomega, \, \bomega^\prime \in \Omega$,
 \begin{align*}
     \|f_{\bomega} - f_{\bomega^\prime}\|_{\fL_2(\lambda)}^2 \ge  \, \underline{b}_\lambda \|f_{\bomega} - f_{\bomega^\prime}\|_{\fL_2(\text{Leb})}^2 =  \|\bomega - \bomega^\prime\|_1 \int h^2_\delta(x) dx = & \|\bomega - \bomega^\prime\|_1 \times a^2 \delta^{2\beta + d} \int b^2(x) dx\\
     \succsim & m^d \delta^{2\beta + d}\\
     = & \delta^{2\beta}.
 \end{align*}
 Let $P_{\bomega}$ denote the the distribution of the form \eqref{e1} with $f_0$ replaced with $f_{\bomega}$. Thus, 
 \begingroup
 \allowdisplaybreaks
 \begin{align}
     \operatorname{KL}(P_{\bomega}^{\otimes_n}\|P_{\bomega^\prime}^{\otimes_n}) =  n \operatorname{KL}(P_{\bomega}\|P_{\bomega^\prime}) 
     = & n \E_{\bx} d_\phi\left(f_{\bomega}(\bx) \| f_{\bomega^\prime}(\bx)\right) \label{e_19.1}\\
     \le & n \tau_2 \bar{b}_\lambda \|\bomega - \bomega^\prime\|_1 \times a^2 \delta^{2\beta + d} \int b^2(x) dx \nonumber\\
     \precsim & n m^d \delta^{2\beta + d}. \nonumber
 \end{align}
 \endgroup
 Here, \eqref{e_19.1} follows from Lemma~\ref{lem_kl}.
 Choosing $m \asymp n^{1/(2\beta + d)}$, we can make, $\operatorname{KL}(P_{\bomega}^{\otimes_n}\|P_{\bomega^\prime}^{\otimes_n}) \le  \frac{m^d}{1000}$. Thus, from \citet[equation 15.34]{wainwright_2019}, \( I(Z;J) \le \frac{1}{M^2} \sum_{\bomega, \bomega^\prime \in \Omega} \operatorname{KL}(P_{\bomega}^{\otimes_n}\|P_{\bomega^\prime}^{\otimes_n}) \le \frac{m^d}{1000}.\) Here $I(Z_1;Z_2)$ denotes the mutual information between the random variables $Z_1$ and $Z_2$ \citep[Section 2.3]{cover2005elements}. Hence, \(\frac{ I(Z;J) + \log 2}{\log M}  \le 8 \frac{m^d/1000 + \log 2}{m^d \log 2} \le 1/2,\) if $n$ is large enough. Thus, applying Proposition 15.2 of \cite{wainwright_2019}, we note that,
 \(\inf_{\hat{f}} \sup_{f \in \sH^\beta(\Real^d, \Real,C)} \E_{f} \|\hat{f} - f\|_{\fL_2(\lambda)}^2 \succsim \delta^{2\beta} \asymp n^{-\frac{2\beta}{2\beta + d}}.\)

\section{Conclusion}\label{sec_con}
In this paper, we discussed a statistical framework to understand the finite sample properties of supervised deep learning for both regression and classification settings. In particular, we modeled the dependence of the response given the explanatory variable through a exponential families and showed that the maximum likelihood estimates can be achieved by minimizing the corresponding Bregman loss and incorporating the mean function as the activation for the final layer. Under the assumption of the existence of a bounded density for the explanatory variable, we showed that deep ReLU networks can achieve the minimax optimal rate when the network size is chosen properly. Furthermore, when the explanatory variable has an intrinsically low dimensional structure, the convergence rate of the sample estimator, in terms of the sample size, only depends on the entropic dimension of the underlying distribution of the explanatory variable, resulting in better convergence rates compared to the existing literature for both classification and regression problems.

While our findings offer insights into the theoretical aspects of deep supervised learning, it is crucial to recognize that assessing the complete error of models in practical applications necessitates the consideration of an optimization error component. Regrettably, the accurate estimation of this component remains a formidable challenge in the non-overparametrized regime due to the non-convex and intricate nature of the optimization problem. Nevertheless, it is worth emphasizing that our error analyses operate independently of the optimization process and can be readily integrated with optimization analyses.

\tableofcontents
\appendix
\section{Proofs of Main Lemmata}
\label{ap1}
\subsection{Proof of Lemma~\ref{lemoct}}
\lemoct*
\begin{proof}

To prove Lemma~\ref{lemoct}, we first make the following observation.
\begingroup
\allowdisplaybreaks
\begin{align}
        & \E d_\phi (y \| \mu(\hat{f}(\bx))) - \E d_\phi (y \| \mu(f_0(\bx))) \nonumber \\
        = &  \E_{\bx} \E_{y|\bx} \bigg( \phi(\mu(f_0(\bx))) - \phi(\mu(\hat{f}(\bx))) - \left \langle \nabla \phi(\mu(\hat{f}(\bx))), y - \mu(\hat{f}(\bx)) \right\rangle  + \left \langle \nabla \phi(\mu(f_0(\bx))), y - \mu(f_0(\bx)) \right\rangle \bigg) \nonumber\\
        = & \E_{\bx} d_\phi \left( \mu(f_0(\bx)) \| \mu(\hat{f}(\bx))\right) \nonumber \\
        \le & \tau_2 \E_{\bx} \| \mu(f_0(\bx)) - \mu(\hat{f}(\bx)) \|_2^2 \label{e_s_3}\\
        \le & \tau_2 \sigma_1 \E_{\bx}  \| f_0(\bx) - \hat{f}(\bx) \|_2^2 \label{e_s_4}\\
        = & \frac{\sigma_1}{\sigma_2} \| f_0 - \hat{f} \|_{\fL_2(\lambda)}^2. \nonumber
    \end{align}
\endgroup    
Here \eqref{e_s_3} follows from Lemma~\ref{lem_tel}.
Inequality \eqref{e_s_4} follows from the fact that $\mu(\cdot)$ is $\sigma_1$-Lipschitz. We also note that,
\begingroup
\allowdisplaybreaks
\begin{align}
         \E d_\phi (y \| \hat{f}(\bx)) - \E d_\phi (y \| \mu(f_0(\bx))) = & \E_{\bx} d_\phi \left( \mu(f_0(\bx)) \| \mu(\hat{f}(\bx))\right) \nonumber \\
        \ge & \tau_1 \E_{\bx} \| \mu(f_0(\bx)) - \mu(\hat{f}(\bx)) \|_2^2 \label{e_s_7}\\
        \ge & \tau_1 \sigma_2 \E_{\bx}  \| f_0(\bx) - \hat{f}(\bx) \|_2^2 \label{e_s_8}\\
        = & \frac{\sigma_2}{\sigma_1} \| f_0 - \hat{f} \|_{\fL_2(\lambda)}^2. \nonumber
    \end{align}
\endgroup  
As before, \eqref{e_s_7} follows from the fact that $\phi$ is $\tau_1$-strongly convex and applying Lemma~\ref{lem_tel}. 
Inequality \eqref{e_s_8} follows from the fact that $\mu^\prime(\cdot) = \Psi^{\prime \prime}(\cdot) \ge \sigma_2 $, due to the strong convexity of $\Psi$ and a simple application of the mean value theorem.
\end{proof}
\subsection{Proof of Lemma~\ref{lem12}}
\lemtwelve*
\begin{proof}
    We note that $ \bar{d}_\alpha(\gamma) \le \overline{\text{dim}}_M(\gamma) = \tilde{d}$, by \citet[Proposition 7]{weed2019sharp}. When $\mu \ll \fH^{\tilde{d}}$, again by \citet[Proposition 8]{weed2019sharp}, it is known that $d_\ast(\gamma) = \tilde{d}$, where $d_\ast(\gamma)$ denotes the lower Wasserstein dimension of $\gamma$ \citep[Definition 4]{weed2019sharp}. The result now follows from Proposition 8 of \citet{chakraborty2024statistical}.
\end{proof}

\subsection{Proof of Lemma~\ref{lem_oracle}}
\lemoracle*
\begin{proof}
    Since $\hat{f}$ is the global minimizer of $\sum_{i=1}^n d_\phi(y_i \| \mu(f(\bx_i)))$, we note that
\begin{align}
    \sum_{i=1}^n d_\phi(y_i \| \mu(\hat{f}(\bx_i))) \le \sum_{i=1}^n d_\phi(y_i \| \mu(f(\bx_i))), \label{19.1}
\end{align}
for any $f \in \cF$. A little algebra shows that \eqref{19.1} is equivalent to
\begin{align}
    & \frac{1}{n}\sum_{i=1}^n d_\phi(\mu(f_0(\bx_i)) \| \mu(\hat{f}(\bx_i))) \nonumber\\
    \le & \frac{1}{n}\sum_{i=1}^n d_\phi(\mu(f_0(\bx_i)) \| \mu(f(\bx_i))) + \frac{1}{n} \sum_{i=1}^n \left \langle \nabla \phi(\mu(\hat{f}(\bx_i))) - \nabla \phi(\mu(f(\bx_i))),  \xi_i \right \rangle. \label{19.2}
\end{align}

From \eqref{19.2}, applying Lemma~\ref{lem_tel},
we observe that
\begin{align}
    & \frac{\tau_1}{n}\sum_{i=1}^n (\mu(f_0(\bx_i)) - \mu(\hat{f}(\bx_i)))^2 \nonumber \\
    \le & \frac{\tau_2}{n}\sum_{i=1}^n (\mu(f_0(\bx_i)) - \mu(f(\bx_i)))^2 + \frac{1}{n} \sum_{i=1}^n \left \langle \nabla \phi(\mu(\hat{f}(\bx_i))) - \nabla \phi(\mu(f(\bx_i))),  \xi_i \right \rangle .
\end{align}
Plugging in $f \leftarrow f^\ast$, we get the desired result.
\end{proof}
\subsection{Proof of Lemma~\ref{lem_17.3}}
\lemel*
\begin{proof}
From the definition of $\sG_\delta$, it is clear that $\log \cN(\epsilon; \sG_\delta, \|\cdot\|_{\infty,n}) \le 2 \log \cN(\epsilon/2 ; \cF, \|\cdot\|_{\infty,n}). $ 
Let $Z_f = \frac{1}{\sqrt{n}} \sum_{i=1}^n \xi_i f(x_i)$. Clearly, $\E_\xi Z_f = 0$. Furthermore, applying Lemma~\ref{lem20.1}, 
we observe that
\begingroup
\allowdisplaybreaks
\begin{align*}
    \E_\xi \exp(\lambda (Z_f-Z_g) ) =  \E_\xi \exp\left( \frac{\lambda}{\sqrt{n}} \sum_{i=1}^n \xi_i (f(x_i) -g (x_i))\right)
    = & \exp\left( \frac{\lambda^2}{2} \|f - g\|_{\fL_2(\lambda_n)}^2 \sigma_1\right).
\end{align*}
\endgroup
Thus, $(Z_f-Z_g)$ is $ \|f - g\|_{\fL_2(\lambda_n)}^2 \sigma_1$-subGaussian. Furthermore, 
\begingroup
\allowdisplaybreaks
\begin{align*}
    \sup_{f,g \in \sG_\delta} \|f - g\|_{\fL_2(\lambda_n)} = & \sup_{f,f^\prime \in \cF: \|f - f^\prime\|_{\fL_\infty(\lambda_n)} \le \delta} \| \nabla \phi(\mu(f)) - \nabla \phi(\mu(f^\prime))\|_{\fL_2(\lambda_n)}\\
    \le & \tau_1 \sigma_1 \sup_{f,f^\prime \in \cF: \|f - f^\prime\|_{\fL_\infty(\lambda_n)} \le \delta} \| f - f^\prime\|_{\fL_2(\lambda_n)}\\
    \le & \sup_{f,f^\prime \in \cF: \|f - f^\prime\|_{\fL_\infty(\lambda_n)} \le \delta} \| f - f^\prime\|_{\fL_\infty(\lambda_n)}\\
    \le &  \delta.
\end{align*}
\endgroup

From \citet[Proposition 5.22]{wainwright_2019}, 
\begingroup
\allowdisplaybreaks
\begin{align}
    \E_\xi \sup_{g \in \sG_\delta} \frac{1}{\sqrt{n}}\sum_{i=1}^n \xi_i g(x_i) =  \E_\xi \sup_{g \in \sG_\delta} Z_g = & \E_\xi \sup_{g \in \sG_\delta} (Z_g - Z_{g^\prime}) \nonumber \\
    \le & \E_\xi \sup_{g, g^\prime \in \sG_\delta} (Z_g - Z_{g^\prime}) \nonumber \\
    \le & 32 \int_0^{ \delta} \sqrt{ \log \cN(\epsilon; \sG_\delta, \fL_2(\lambda_n) )} d\epsilon \nonumber \\
    \precsim &  \int_0^{\delta} \sqrt{\log \cN(\epsilon/(2\sigma_1); \cF, \fL_\infty(\lambda_n) )} d\epsilon \nonumber \\
    \precsim & \int_0^{\delta} \sqrt{\operatorname{Pdim}(\cF) \log (n/\epsilon)  } d\epsilon \nonumber \\
    \le  &  \delta \sqrt{\operatorname{Pdim}(\cF) \log n} + \sqrt{\operatorname{Pdim}(\cF)}\int_0^{ \delta}  \sqrt{\log (1/\epsilon)}   d\epsilon \nonumber \\
    \le & \delta \sqrt{\operatorname{Pdim}(\cF) \log n} + 2 \sqrt{\operatorname{Pdim}(\cF)}  \delta \sqrt{ \log( 1/\delta)} \label{e12}\\
    \precsim & \delta \sqrt{\operatorname{Pdim}(\cF) \log(n/\delta)}.
    \end{align}
\endgroup
Here, \eqref{e12} follows from Lemma~\ref{lem_17.1}. Thus, \( \E \sup_{g \in \sG_\delta} \frac{1}{n}\sum_{i=1}^n \xi_i g(x_i) \precsim   \delta \sqrt{\frac{\operatorname{Pdim}(\cF) \log (n/\delta)}{n}}.\) Applying Lemma \ref{lem:1}, we note that for $t>0$, with probability at least $1 -  e^{-nt^2/\delta^2}$,
\begin{align}
    \sup_{g \in \sG_\delta} \frac{1}{n}\sum_{i=1}^n \xi_i g(x_i) \precsim &  t + \delta \sqrt{\frac{\operatorname{Pdim}(\cF) \log (n/\delta)}{n}}.\label{e10}
\end{align}
\end{proof}

\subsection{Proof of Lemma~\ref{lem_17.2}}
\lemgentwo*
\begin{proof}
    We take $\delta = \max\left\{n^{-\alpha}, 2\|\hat{f} - f_0\|_{\fL_2(\lambda_n)}\right\} $ and let $t = n^{-2 \alpha}$. We consider two cases as follows.%, where $\alpha_1\ge 1$ is an absolute constant chosen later 
    
\textbf{Case 1}: $\|\hat{f} - f^\ast\|_{\fL_2(\lambda_n)} \le \delta$.

Then, by Lemma~\ref{lem_17.3}, with probability at least $1 -  \exp\left(-n^{1-2\alpha} \right)$,
\begingroup
\allowdisplaybreaks
\begin{align}
    \|\hat{f} - f^\ast\|^2_{\fL_2(\lambda_n)} \le & 2 \|\hat{f} - f_0\|^2_{\fL_2(\lambda_n)} + 2 \|f_0 - f^\ast\|^2_{\fL_2(\lambda_n)} \nonumber\\
    \precsim &   \|f_0 - f^\ast\|^2_{\fL_2(\lambda_n)} + \|\mu(\hat{f}) - \mu(f_0)\|^2_{\fL_2(\lambda_n)} \label{es13}\\
    \precsim & \|f_0 - f^\ast\|^2_{\fL_2(\lambda_n)} + \sup_{g \in \sG_\delta} \frac{1}{n} \sum_{i=1}^n \xi_i g(x_i) \label{e14}\\
    \precsim & \|f_0 - f^\ast\|^2_{\fL_2(\lambda_n)} + t + \delta \sqrt{\frac{\operatorname{Pdim}(\cF) \log (n/\delta)}{n}} 
    \label{a_23.1}
\end{align}
\endgroup
In the above calculations, \eqref{es13} follows from the fact that $\mu(\cdot)$ is strongly convex and \eqref{e14} follows from Lemma~\ref{lem_oracle}. Inequality \eqref{a_23.1} follows from Lemma~\ref{lem_17.3}. Let $\alpha_1 \ge 1$ be the corresponding constant that honors the inequality in \eqref{a_23.1}. Then using the upper bound on $\delta$, we observe that
\begingroup
\allowdisplaybreaks
\begin{align}
  &  \|\hat{f} - f^\ast\|^2_{\fL_2(\lambda_n)} \nonumber\\
  \le & \alpha_1 \|f_0 - f^\ast\|^2_{\fL_2(\lambda_n)}  + \alpha_1 \delta \sqrt{\frac{\operatorname{Pdim}(\cF) \log (n/\delta)}{n}} + n^{-2\alpha }\nonumber\\
  \le & \alpha_1 \|f_0 - f^\ast\|^2_{\fL_2(\lambda_n)} + \frac{\delta^2 }{16} +  \frac{4 \alpha_1^2 }{n}\operatorname{Pdim}(\cF) \log (n/\delta) + n^{-2\alpha } \label{e15}\\
    \le & \alpha_1 \|f_0 - f^\ast\|^2_{\fL_2(\lambda_n)} + \frac{9n^{-2\alpha}}{8}  + \frac{1}{4} \|\hat{f}- f_0\|^2_{\fL_2(\lambda_n)} +  \frac{4(1 + \alpha)\alpha_1^2}{n }\operatorname{Pdim}(\cF) \log (n) \nonumber \\
    \le & \alpha_1 \|f_0 - f^\ast\|^2_{\fL_2(\lambda_n)} + 2  n^{-2\alpha} + \frac{1}{2} \|\hat{f}- f^\ast\|^2_{\fL_2(\lambda_n)} + \frac{1}{2} \|f^\ast - f_0\|^2_{\fL_2(\lambda_n)} +  \frac{4(1 + \alpha)\alpha_1^2 }{n }\operatorname{Pdim}(\cF) \log (n). \nonumber
\end{align}
\endgroup
Here, \eqref{e15} follows from the fact that $\sqrt{xy} \le \frac{x}{16\alpha_1 } + 4 \alpha_1 y$, from the AM-GM inequality and taking $x = \delta^2$ and $y = \frac{\operatorname{Pdim}(\cF) \log (n/\delta)}{n}$. Thus,
\begin{align}\label{eq_26}
   \|\hat{f} - f^\ast\|^2_{\fL_2(\lambda_n)} \precsim & n^{-2\alpha} +\|f^\ast - f_0\|^2_{\fL_2(\lambda_n)} +  \frac{1}{n }\operatorname{Pdim}(\cF) \log (n) .
\end{align}
\textbf{Case 2}: $\|\hat{f} - f^\ast\|_{\fL_2(\lambda_n)} \ge \delta$.

It this case, we note that $\|\hat{f} -f^\ast\|_{\fL_2(\lambda_n)} \ge 2 \|\hat{f} - f_0\|_{\fL_2(\lambda_n)}$. Thus,
\begin{align*}
    \|\hat{f} - f^\ast\|_{\fL_2(\lambda_n)}^2 \le & 2 \|\hat{f} - f_0\|_{\fL_2(\lambda_n)}^2 + 2 \|f_0 - f^\ast\|_{\fL_2(\lambda_n)}^2 \le \frac{1}{2} \|\hat{f} - f^\ast\|_{\fL_2(\lambda_n)}^2 + 2 \|f_0 - f^\ast\|_{\fL_2(\lambda_n)}^2 \\
    \implies & \|\hat{f} - f^\ast\|_{\fL_2(\lambda_n)}^2 \precsim \|f_0 - f^\ast\|_{\fL_2(\lambda_n)}^2.
\end{align*}
Thus, from the above two cases, combining equations \eqref{a_23.1} and \eqref{eq_26}, with probability at least, $1 -  \exp\left(-n^{1-2\alpha}\right)$, 
\begin{equation}
    \|\hat{f} - f^\ast\|^2_{\fL_2(\lambda_n)} \precsim  n^{-2 \alpha} +\|f^\ast - f_0\|^2_{\fL_2(\lambda_n)} +  \frac{1}{n }\operatorname{Pdim}(\cF) \log (n). \label{e_s3}
\end{equation}

From equation \eqref{e_s3}, we note that, for some constant $B_4$, 
\begin{align*}
    \prob\left( \|\hat{f} - f^\ast\|^2_{\fL_2(\lambda_n)} \le B_4\left(  n^{-2 \alpha} +\|f^\ast - f_0\|^2_{\fL_2(\lambda_n)} +  \frac{1}{n }\operatorname{Pdim}(\cF) \log (n) \right)\bigg| x_{1:n}\right) \ge 1 -  \exp\left(-n^{1-2\alpha}\right).
\end{align*}
Integrating both sides w.r.t. the measure $\mu^{\otimes_n}$, i.e. the joint distribution of $\bx_{1:n}$, we observe that, unconditionally, with probability at least $1 -  \exp\left(-n^{1-2 \alpha}\right)$, 
\begin{equation}
    \|\hat{f} - f^\ast\|^2_{\fL_2(\lambda_n)} \precsim  n^{-2 \alpha} +\|f^\ast - f_0\|^2_{\fL_2(\lambda_n)} +  \frac{1}{n }\operatorname{Pdim}(\cF) \log (n). \label{e_s5}
\end{equation}
\end{proof}

\subsection{Proof of Lemma~\ref{thm_17.1}}
To prove Lemma \ref{thm_17.1}, we first state and prove the following result. 
\begin{lem}\label{lem_17.4}
     For $\alpha \in (0,1/2)$, if $n \ge \max\left\{e^{1/\alpha}, \operatorname{Pdim}(\cF)\right\}$, with probability at least $1 - 2 \exp\left(-n^{1-2\alpha}\right)$,
    \begin{align*}
    \|\hat{f} - f^\ast\|^2_{\fL_2(\lambda)} \precsim & n^{-2\alpha} +\|f^\ast - f_0\|^2_{\fL_2(\lambda_n)}+  \frac{1}{n }\operatorname{Pdim}(\cF) \log^2 n + \frac{\log \log n}{n}.
\end{align*}
\end{lem}

\begin{proof}
From Lemma~\ref{lem_bd_rad}, we note that if $n \ge \operatorname{Pdim}(\cF)$, then, 
\begin{align}
 \E_{\epsilon} \sup_{h \in \cH_r} \frac{1}{n}\sum_{i=1}^n \epsilon_i h(\bx_i) \precsim & \sqrt{\frac{ r\operatorname{Pdim}(\cF) \log n}{n}} \label{e_f_4_1}\\
     \le & \sqrt{\frac{(\operatorname{Pdim}(\cF))^2 \log n}{n^2} + r \frac{\operatorname{Pdim}(\cF) \log(n/e\operatorname{Pdim}(\cF)) \log n}{n}}. \label{e13}
\end{align}
Here, \eqref{e_f_4_1} follows from Lemma~\ref{lem_bd_rad} and \eqref{e13} follows from the fact that for all $x,y >0$, $x \log x \le y + x \log(1/ye)$. %Applying Sridharacharya's formula, 
It is easy to see that the RHS of \eqref{e13} has a fixed point of $r^\ast$ and $r^\ast \precsim \frac{\operatorname{Pdim}(\cF) \log^2 n}{n}$. Then, by Theorem 6.1 of \cite{1444}, we note that with probability at least $1 - e^{-x}$,
\begin{equation}
    \int h d\lambda  \le B_3\left(\int h d\lambda_n + \frac{\operatorname{Pdim}(\cF) \log^2 n}{n} + \frac{x}{n} + \frac{\log \log n}{n}\right), \forall h \in \cH, \label{e_s4}
\end{equation}
for some absolute constant $B_3$. Now, taking $x = n^{1 - 2\alpha}$ in \eqref{e_s4}, we note that, with probability at least $1 - \exp\left(-n^{1-2\alpha}\right)$,
\begin{align}
    \|\hat{f} - f^\ast\|^2_{\fL_2(\lambda)} \precsim n^{-2\alpha} +\|\hat{f} - f^\ast\|^2_{\fL_2(\lambda_n)}+  \frac{1}{n }\operatorname{Pdim}(\cF) \log^2 n + \frac{\log \log n}{n}. \label{e_s7}
\end{align}
Combining \eqref{e_s7} with Lemma~\ref{lem_17.2}, we observe that with probability at least $1- 2 \exp\left(-n^{1-2\alpha}\right)$,
\begin{equation}
    \|\hat{f} - f^\ast\|^2_{\fL_2(\lambda)} \precsim  n^{-2 \alpha} +\|f^\ast - f_0\|^2_{\fL_2(\lambda_n)} +  \frac{1}{n }\operatorname{Pdim}(\cF) \log^2 n + \frac{\log \log n}{n}.
\end{equation}
\end{proof}

\subsubsection{Proof of Lemma~\ref{thm_17.1}}
\lemgenone*
\begin{proof}
    Let $Z_i = (f^\ast(\bx_i) - f_0(\bx_i))^2 - \E (f^\ast(\bx_i) - f_0(\bx_i))^2$. Since $\|f_0\|_\infty, \|f^\ast\|_\infty \le B$, for some constant $B$, it is easy to see that
\begin{align*}
    \sigma^2 = \sum_{i=1}^n \E Z_i^2    =   \sum_{i=1}^n \operatorname{Var}\left( (f^\ast(\bx_i) - f_0(\bx_i))^2 \right)
    \le   \sum_{i=1}^n \E \left( f^\ast(\bx_i) - f_0(\bx_i)\right)^4
    \le & 4nB^2 \|f^\ast - f_0\|^2_{\fL_2(\lambda)}.
\end{align*}
Taking $u = v \vee \|f^\ast - f_0\|^2_{\fL_2(\lambda)}$, we note that, $\sigma^2 \le 4nB^2 u$. Applying Bernstein's inequality, e.g., \citep[Theorem~2.8.4]{vershynin2018high}, %(Lemma~\ref{bernstein}), 
with $t = n u$,
\begingroup
\allowdisplaybreaks
{\small
\begin{align*}
    \prob \left(\left| \|f^\ast - f_0\|_{\fL(\lambda_n)}^2 - \|f^\ast - f_0\|_{\fL(\lambda)}^2 \right| \ge u\right) \le  \exp\left(-\frac{n^2 u^2/2}{4nB^2 u + Bn u/3}\right)
    \le & \exp\left(-\frac{n u}{8B^2 + 2B/3}\right)
    =  \exp\left(-\frac{n u}{B_5}\right) \le \exp\left(-\frac{n v}{B_5}\right)
\end{align*}
}%
\endgroup
with $B_5 = 8B^2 + 2B/3$. Thus, with probability, at least $1 - \exp\left(-\frac{n v}{B_5}\right)$,
\[\|f^\ast - f_0\|_{\fL(\lambda_n)}^2 \le  \|f^\ast - f_0\|_{\fL(\lambda)}^2 + u \le 2\|f^\ast - f_0\|_{\fL(\lambda)}^2 + v .\]
Taking $v = B_5 n^{-2\alpha}$, we observe that, with probability at least $1 - \exp\left(-n^{1-2\alpha}\right)$,
\begin{equation}
    \label{e_s6}
    \|f^\ast - f_0\|_{\fL(\lambda_n)}^2 \precsim \|f^\ast - f_0\|_{\fL(\lambda)}^2 + n^{-2 \alpha}.
\end{equation}
Combining \eqref{e_s6} with Lemma~\ref{lem_17.4}, we observe that, with probability at least $1 - 3 \exp\left(-n^{1-2\alpha}\right)$
\begin{align*}
    \|\hat{f} - f^\ast\|^2_{\fL_2(\lambda)} \precsim n^{-2 \alpha} +\|f^\ast - f_0\|^2_{\fL_2(\lambda)}+  \frac{1}{n }\operatorname{Pdim}(\cF) \log^2 n + \frac{\log \log n}{n}.
\end{align*}
The theorem now follows from observing that $\|\hat{f} - f_0\|^2_{\fL_2(\lambda)} \le 2 \|\hat{f} - f^\ast\|^2_{\fL_2(\lambda)} + 2 \| f_0 - f^\ast\|^2_{\fL_2(\lambda)}$.
\end{proof}

\section{Proof of Approximation Results (Lemma~\ref{lem_approx})}\label{pf_approx}
\lemapprox*
\begin{proof}
We first fix any $\epsilon \in (0,1)$ and let, $K = \lceil \frac{1}{2\epsilon}\rceil$. For any $\bi \in [K]^d$, let $\btheta^{\bi} = ( \epsilon + 2 (i_1-1) \epsilon, \dots,  \epsilon + 2(i_d-1) \epsilon )$. Clearly, $\{\btheta^{\bi}: \bi \in [K]^d\}$ constitutes an $\epsilon$-net of $[0,1]^d$, w.r.t. the $\ell_\infty$-norm. We let
\[\xi_{a, b}(x) = \relu\left(\frac{x+a}{a-b}\right) - \relu\left(\frac{x+b}{a-b}\right) - \relu\left(\frac{x-b}{a-b}\right) + \relu\left(\frac{x-a}{a-b}\right),\]
for any $0 <b \le a $. For $0 < \delta \le \epsilon/3$, we define
\[\zeta_{\epsilon, \delta}(\bx) = \prod_{j = 1}^d \xi_{\epsilon, \delta}(x_j).\]
We define the region $\sQ_{\epsilon,\delta} = \cup_{\bi \in [K]^d} B_{\ell_\infty} (\btheta^{\bi}, \delta)$. It is easy to observe that $\operatorname{Leb}([0,1]^d \setminus \sQ_{\epsilon, \delta}) \le 2 d \delta $. Here $\operatorname{Leb}(\cdot)$ denotes the Lebesgue measure on $\Real^d$. Clearly, $\zeta_{\epsilon, \delta} \left(\cdot - \btheta^{\bi}\right) = 1 $ on $\sQ_{\epsilon, \delta}$. 

Consider the Taylor expansion of $f$ around $\btheta$ as,
\(P_{\btheta}(\bx) = \sum_{|\bs| \le \lfloor \beta \rfloor} \frac{\partial^{\bs} f(\btheta)}{ \bs !} \left(\bx - \btheta\right)^{\bs} .\)
\begin{align}
   \text{Clearly}, \, f(\bx) - P_{\btheta}(\bx) =  \sum_{|\bs| = \lfloor \beta \rfloor} \frac{\left(\bx - \btheta\right)^{\bs}}{ \bs !} \left(\partial^{\bs} f(\by) - \partial^{\bs} f(\btheta)\right) 
    \le & \|\bx - \btheta\|_{\infty}^{\lfloor \beta \rfloor} \sum_{|\bs| = \lfloor \beta \rfloor} \frac{1}{ \bs !} \left|\partial^{\bs} f(\by) - \partial^{\bs} f(\btheta)\right| \nonumber\\
    \le & \frac{C d^{\lfloor \beta \rfloor}}{\lfloor \beta \rfloor !} \|\bx - \btheta\|_\infty^{\beta}. \label{e4}
\end{align}
In the above calculations, $\by$ lies on the line segment joining $\bx$ and $\btheta$. Inequality \eqref{e4} follows from the fact that $\left|\partial^{\bs} f(\by) - \partial^{\bs} f(\btheta)\right| 
 \le C \|\by - \btheta\|_{\infty}^{\beta - \lfloor \beta \rfloor} \le C \|\bx - \btheta\|_{\infty}^{\beta - \lfloor \beta \rfloor}$ and the identity $\frac{d^k}{k!} = \sum_{|\bs| = k} \frac{1}{\bs!}$.
 Next, we suppose that $\tilde{f}(\bx) = \sum_{\bi \in [K]^d} \zeta_{\epsilon, \delta}(\bx - \btheta^{\bi}) P_{\btheta^{\bi}}(\bx)$. Thus, if $\bx \in \sQ_{\epsilon, \delta}$,
 \begin{align}
     |f(\bx) - \tilde{f}(\bx)| \le  \max_{\bi \in [K]^d} \sup_{\bx \in B_{\ell_\infty}(\btheta^{\bi}, \delta)}|f(\bx) - P_{\btheta^{\bi}}(\bx)|
     \le & \frac{C d^{\lfloor \beta \rfloor}}{\lfloor \beta \rfloor !} \delta^\beta. \label{e5}
 \end{align}
 Here, \eqref{e5} follows from \eqref{e4}. Thus, $\|f - \tilde{f}\|_{\fL_\infty(\sQ_{\epsilon, \delta})} \le \frac{C d^{\lfloor \beta \rfloor}}{\lfloor \beta \rfloor !} \delta^\beta $. Furthermore, by definition, $\|\tilde{f}\|_{\infty} \le C + \frac{C d^{\lfloor \beta \rfloor}}{\lfloor \beta \rfloor !} \epsilon^\beta $. Let $a_{\bi, \bs} =  \frac{\partial^{\bs} f(\theta^{\bi})}{ \bs !}$ and 
 {\small
 \begin{align*}
     \hat{f}_{\bi, \bs}(\bx) = \text{prod}_m^{(d+|\bs|)}( & \xi_{\epsilon_1, \delta_1}(x_1 - \theta_1^{\bi}), \dots, \xi_{\epsilon_d, \delta_d}(x_d - \theta_d^{\bi}), \underbrace{(x_1-\theta_1^{\bi}), \dots, (x_1-\theta_1^{\bi})}_{\text{$s_1$ times}}, \dots, \underbrace{(x_1-\theta_d^{\bi}), \dots ,(x_d -\theta_d^{\bi})}_{\text{$s_d$ times}}).
 \end{align*}
 }%
 Here, $\text{prod}_m^{(d+|\bs|)}$ is an approximation of the product function developed by \cite{chakraborty2024statistical} (see Lemma~\ref{s.5}) and has at most $d + |\bs| \le d + \lfloor \beta \rfloor$ many inputs. By \citet[Lemma 40]{chakraborty2024statistical},
 $\text{prod}_m^{(d+|\bs|)}$ can be implemented by a ReLU network with $\cL(\text{prod}_m^{(d+|\bs|)})$, $ \cW(\text{prod}_m^{(d+|\bs|)}) \le c_3 m$, where $c_3$ is an absolute constant. Thus, $\cL(\hat{f}_{\bi, \bs}) \le c_3 m + 2$ and $\cW(\hat{f}_{\bi, \bs}) \le c_3 m + 8 d + 4 |s| \le c_3 m+ 8 d + 4 \lfloor \beta \rfloor $.
From \citet[Lemma 40]{chakraborty2024statistical}, %~\ref{lem_43},
we observe that, 
\begin{align}
    \left|\hat{f}_{\bi, \bs}(\bx) - \zeta(x - \theta^{\bi})  \left(x - \theta^{\bi}\right)^{\bs}\right| \le \frac{1}{2^m}, \, \forall x \in S. \label{es6}
\end{align}
Here, $m \ge \max \frac{1}{2} (\log_2 (4d) - 1)$. Finally, let $\hat{f}(\bx) = \sum_{\bi \in [K]^d} \sum_{|\bs| \le \lfloor \beta \rfloor} a_{\bi, \bs} \hat{f}_{\bi, \bs}(\bx)$. Clearly, $\mathcal{L}(\hat{f}) \le c_3 m + 3$ and $\mathcal{W}(\hat{f}) \le \binom{d+\lfloor \beta \rfloor}{\lfloor \beta \rfloor} \left(c_3 m + 8 d + 4 \lfloor \beta \rfloor \right)$. This implies that, 
\begingroup
\allowdisplaybreaks
\begin{align}
    \sup_{\bx \in \sQ_{\epsilon,\delta} } |\hat{f}(\bx) - \tilde{f}(\bx)| \le & \max_{\bi \in [K]^d} \sup_{\bx \in B_{\ell_\infty}(\btheta^{\bi}, \delta)} \sum_{|\bs| \le \lfloor \beta \rfloor}  |a_{\bi, \bs} | \zeta(x - \theta^{\bi})  |\hat{f}_{\bi \bs}(x) -  \left(x - \theta^{\bi}\right)^{\bs}|\nonumber \\
    \le &   \sum_{|\bs| \le k}  |a_{\theta, \bs} | \left|\hat{f}_{\theta^{\bi(x)}, \bs}(x) - \zeta_{\bepsilon, \bdelta}(x - \theta^{\bi(x)})  \left(x - \theta^{\bi(x)}\right)^{\bs}\right| \nonumber\\
    \le & \frac{  C}{2^m}. \label{e7}
\end{align}
\endgroup
From \eqref{e5} and \eqref{e7}, we thus get that if $\bx \in \sQ_{\epsilon, \delta}$,
\begin{align}
    |f(\bx) - \hat{f}(\bx)| \le & |f(\bx) - \tilde{f}(\bx)| + |\hat{f}(\bx) - \tilde{f}(\bx)| \le  \frac{C d^{\lfloor \beta \rfloor}}{\lfloor \beta \rfloor !} \delta^\beta + \frac{C}{2^m}.\label{e8}
\end{align}
Furthermore, it is easy to observe that, $\|\hat{f}\|_{\fL_\infty([0,1]^d)} \le C + \frac{C d^{\lfloor \beta \rfloor}}{\lfloor \beta \rfloor !} \epsilon^\beta + \frac{C}{2^m}$. Hence, 
\begin{align*}
    \|f - \hat{f}\|_{\fL_p(\text{Leb})}^p = &\int_{\sQ_{\epsilon, \delta}} |f(\bx) - \hat{f}(\bx)|^p d\text{Leb}(\bx) + \int_{\sQ_{\epsilon, \delta}^\complement } |f(\bx) - \hat{f}(\bx)|^p d\text{Leb}(\bx)\\
    \le &  \left(\frac{C d^{\lfloor \beta \rfloor}}{\lfloor \beta \rfloor !} \delta^\beta + \frac{C}{2^m}\right)^p \operatorname{Leb}\left(\sQ_{\epsilon, \delta}\right) + \left(2 C +\frac{C d^{\lfloor \beta \rfloor}}{\lfloor \beta \rfloor !} \epsilon^\beta + \frac{C}{2^m}\right)^p \operatorname{Leb}(\sQ_{\epsilon, \delta}^\complement)\\
   \le & \left(\frac{C d^{\lfloor \beta \rfloor}}{\lfloor \beta \rfloor !} \delta^\beta + \frac{C}{2^m}\right)^p  + 2 \left(2 C +\frac{C d^{\lfloor \beta \rfloor}}{\lfloor \beta \rfloor !} \epsilon^\beta + \frac{C}{2^m}\right)^p d \delta\\
   \implies  \|f - \hat{f}\|_{\fL_p(\text{Leb})} \le& \frac{2 C d^{\lfloor \beta \rfloor}}{\lfloor \beta \rfloor !} \delta^\beta + \frac{2C}{2^m}  + 4 C (d \delta)^{1/p} 
   \le  \frac{2 C d^{\lfloor \beta \rfloor}}{\lfloor \beta \rfloor !} \epsilon^\beta + \frac{2C}{2^m}  + 4 C   \epsilon^{\beta}
   \le   10 C d^{\lfloor \beta \rfloor} \epsilon^\beta + \frac{2C}{2^m},
\end{align*}
taking $\delta = \frac{1}{d} \epsilon^{p \beta } \wedge (\epsilon/3)$. We take $m = \left\lceil \log_2 \left(\frac{8}{\eta d^{\lfloor \beta \rfloor}}\right) \right\rceil$ and $\epsilon = (\eta/20)^{1/\beta}$. Thus,
\(\|f - \hat{f}\|_{\fL_p(\text{Leb})} \le C d^{\lfloor \beta \rfloor} \eta.\)
We note that $\hat{f}$ has at most $K^d$-many networks of depth $c_3 m + 3$ and number of weights $ {{d+\lfloor \beta \rfloor} \choose \lfloor \beta \rfloor} \left(c_3 m + 8 d + 4 \lfloor \beta \rfloor \right)$. Thus, 
$\cL(\hat{f}) \le c_3 m + 4$ and $\cW(\hat{f}) \le  K^{d} {{d+\lfloor \beta \rfloor} \choose \lfloor \beta \rfloor} \left(c_3 m + 8 d + 4 \lfloor \beta \rfloor\right)$. We thus get,
\[\cL(\hat{f}) \le c_3 m + 4 \le c_3 \left\lceil \log_2 \left(\frac{8}{\eta d^{\lfloor \beta \rfloor}}\right) \right\rceil + 4.\] 
Similarly,
\begin{align*}
    \cW(\hat{f}) \le & K^d {{d+\lfloor \beta \rfloor} \choose \lfloor \beta \rfloor} \left(c_3 m + 8 d + 4 \lfloor \beta \rfloor \right)\\
    \le & \left\lceil \frac{1}{2(\eta/20)^{1/\beta}}\right\rceil^d {{d+\lfloor \beta \rfloor} \choose \lfloor \beta \rfloor} \left(c_3 \left\lceil \log_2 \left(\frac{8}{\eta d^{\lfloor \beta \rfloor}}\right) \right\rceil + 8 d + 4 \lfloor \beta \rfloor \right)\\
    \le & \left\lceil \frac{1}{2(\eta/20)^{1/\beta}}\right\rceil^d \left(\frac{3}{\beta}\right)^\beta (d+\lfloor\beta \rfloor)^{\lfloor \beta \rfloor} \left(c_3 \left\lceil \log_2 \left(\frac{8}{\eta d^{\lfloor \beta \rfloor}}\right) \right\rceil + 8 d + 4 \lfloor \beta \rfloor \right).
\end{align*}
The proof is now complete by replacing $c_3$ with $\vartheta$.
\end{proof}
\section{Proof of Corollary \ref{cor_1}}
\begin{proof}
Suppose that $\kappa$ be the constant that honors the inequality. Then
    \begin{align*}
        \E \|\hat{f} - f_0\|^2_{\fL_2(\lambda)} = & \E \|\hat{f} - f_0\|^2_{\fL_2(\lambda)} \one \left\{\|\hat{f} - f_0\|^2_{\fL_2(\lambda)} \le \kappa d^{  \frac{2 \lfloor \beta \rfloor( \beta + d)}{2 \beta + d}} n^{-\frac{2\beta}{2\beta + d^\star}} (\log n)^5\right\} \\
         & + \E \|\hat{f} - f_0\|^2_{\fL_2(\lambda)} \one \left\{\|\hat{f} - f_0\|^2_{\fL_2(\lambda)} > \kappa d^{  \frac{2 \lfloor \beta \rfloor( \beta + d)}{2 \beta + d}} n^{-\frac{2\beta}{2\beta + d^\star}} (\log n)^5\right\}\\
        \le & \kappa n^{-\frac{2\beta}{2\beta + d^\star}} (\log n)^5  + 9 C^2 \prob\left(\|\hat{f} - f_0\|^2_{\fL_2(\lambda)} > n^{-\frac{2\beta}{2\beta + d^\star}} (\log n)^5\right)\\
        \le & \kappa n^{-\frac{2\beta}{2\beta + d^\star}} (\log n)^5  + 27 C^2 \exp\left(-n^{\frac{d^\star}{2 \beta + d^\star}}\right) \\
        \precsim & \,\, n^{-\frac{2\beta}{2\beta + d^\star}} (\log n)^5 .
    \end{align*}
\end{proof}
\section{Supporting Lemmas}
\begin{lem}\label{lem_bd_rad}
    Let $\cH_r = \{h = (f - f^\prime)^2 : f, f^\prime \in \cF \text{ and } \lambda_n h \le r\}$ with $\sup_{f\in \cF} \|f\|_{\fL_\infty(\lambda_n)}< \infty$. Then, we can find $r_0>0$, such that if $0<r\le r_0$ and $n \ge \operatorname{Pdim}(\cF)$,  
    \[\E_{\epsilon} \sup_{h \in \cH_r} \frac{1}{n}\sum_{i=1}^n \epsilon_i h(\bx_i) \precsim \sqrt{\frac{r \log(1/r) \operatorname{Pdim}(\cF) \log n}{n}}.\]
\end{lem}
\begin{proof}
  Let $B= 4 \sup_{f\in \cF} \|f\|_{\fL_\infty(\lambda_n)}^2$. We first fix $\epsilon \le \sqrt{2B r}$ and let $h = f - f^\prime$ be a member of $\cH_r$ with $f, f^\prime \in \cF$. We use the notation $\cF_{|_{\bx_{1:n}}} =\{(f(\bx_1), \dots, f(\bx_n))^\top: f \in \cF\}$. Let $\bv^{f}, \bv^{f^\prime} \in \cC(\epsilon; \cF_{|_{\bx_{1:n}}}, \|\cdot\|_\infty)$ be such that $|\bv^f_i - f(\bx_i)|,|\bv^{f^\prime}_i - f^\prime(\bx_i)| \le \epsilon $, for all $i$. Here $\cC(\epsilon; \cF_{|_{\bx_{1:n}}}, \|\cdot\|_\infty)$ denotes the $\epsilon$ cover of $\cF_{|_{\bx_{1:n}}}$ w.r.t. the $\ell_\infty$-norm. Let $\bv = \bv^f - \bv^{f^\prime}$. Then
  \begingroup
  \allowdisplaybreaks
  \begin{align}
       \frac{1}{n}\sum_{i=1}^n (h(\bx_i) - \bv_i^2 )^2 
      = & \frac{1}{n}\sum_{i=1}^n ((f(\bx_i) - f^\prime(\bx_i))^2 - (v_i^f - v_i^{f^\prime})^2 )^2 \nonumber\\
      \le &  \frac{2}{n}\sum_{i=1}^n ((f(\bx_i) - f^\prime(\bx_i))^2 + (v_i^f - v_i^{f^\prime})^2 ) ((f(\bx_i) - f^\prime(\bx_i)) - (v_i^f - v_i^{f^\prime}))^2 \label{e_s1}\\
      \precsim & \epsilon^2.
  \end{align}
  \endgroup
  Here \eqref{e_s1} follows from the fact that $(t^2 - r^2)^2 = (t+r)^2 (t-r)^2 \le 2 (t^2+r^2)(t-r)^2$, for any $t,r \in \Real$.  Hence, from the above calculations, $\cN(\epsilon; \cH_r, \fL_2(\lambda_n)) \le \left(\cN\left( a_1 \epsilon; \cF, \fL_\infty(\lambda_n)\right)\right)^2$, for some absolute constant $a_1$.
  \begingroup
  \allowdisplaybreaks
  %{\small
  \begin{align*}
      \operatorname{diam}^2(\cH_r, \fL_2(\lambda_n)) 
      = \sup_{h, h^\prime \in \cH_r} \frac{1}{n}\sum_{i=1}^n (h(\bx_i) - h^\prime(\bx_i))^2      \le  2  \sup_{h \in \cH_r} \frac{1}{n}\sum_{i=1}^n h^2(\bx_i) 
      \le  2 B \sup_{h \in \cH_r} \frac{1}{n}\sum_{i=1}^n h(\bx_i)
      \le  2 B r.
  \end{align*}
%  }%
  \endgroup
  Hence, $\operatorname{diam}(\cH_r, \fL_2(\lambda_n)) \le  \sqrt{2 B r} $.
  Thus from \citet[Theorem 5.22]{wainwright_2019},
  \begingroup
  \allowdisplaybreaks
\begin{align}
    \E_{\epsilon} \sup_{h \in \cH_r} \frac{1}{n}\sum_{i=1}^n \epsilon_i h(\bx_i)  \precsim & \int_0^{\sqrt{2 B r}} \sqrt{\frac{1}{n} \log \cN(\epsilon; \cH_r, \fL_2(\lambda_n))} d\epsilon \nonumber \\
    \le & \int_0^{\sqrt{2 B r}} \sqrt{\frac{2 \operatorname{Pdim}(\cF)}{n} \log \left(\frac{a_2 n}{\epsilon}\right)} d\epsilon \nonumber \\
    \precsim & \sqrt{2 B r} \sqrt{\frac{\operatorname{Pdim}(\cF) \log n}{n}} + \int_0^{\sqrt{2 B r}} \sqrt{\frac{\operatorname{Pdim}(\cF)}{n} \log (a_2/\epsilon)} d\epsilon \nonumber\\
    \precsim &  \sqrt{\frac{r \log(1/r) \operatorname{Pdim}(\cF) \log n}{n}}. \label{e6}
\end{align} 
\endgroup
Here, \eqref{e6} follows from Lemma~\ref{lem_17.1}. 
\end{proof}
 
\begin{restatable}{lem}{lemkl}
    \label{lem_kl}
    \(\operatorname{KL}(p_{\Psi, \btheta}\| p_{\Psi, \btheta^\prime}) = d_\phi\left(\bmu(\btheta)\|\bmu(\btheta^\prime)\right) .\)
\end{restatable}
\begin{proof}
To prove this result, we make the following observations.
\begingroup
\allowdisplaybreaks
\begin{align}
    & \operatorname{KL}(p_{\Psi, \btheta}\| p_{\Psi, \btheta^\prime}) \nonumber\\
    = & \E_{\bx \sim p_{\Psi, \btheta} } \left(d_\phi(\bx\| \bmu(\btheta^\prime) - d_\phi(\bx\| \bmu(\btheta) \right) \nonumber\\
    = & \E_{\bx \sim p_{\Psi, \btheta} } \left( \phi(\bx) - \phi(\bmu(\btheta^\prime)) - \langle \nabla \phi(\bmu(\btheta^\prime), \bx - \bmu(\btheta^\prime) \rangle  - \left(\phi(\bx) - \phi(\bmu(\btheta)) - \langle \nabla \phi(\bmu(\btheta), \bx - \bmu(\btheta) \rangle\right)\right) \nonumber\\
    = & \E_{\bx \sim p_{\Psi, \btheta} } \left(\phi(\bmu(\btheta)) - \phi(\bmu(\btheta^\prime)) - \langle \nabla \phi(\bmu(\btheta^\prime), \bx - \bmu(\btheta^\prime) \rangle  +  \langle \nabla \phi(\bmu(\btheta), \bx - \bmu(\btheta) \rangle\right)\nonumber\\
    = & d_\phi\left(\bmu(\btheta)\|\bmu(\btheta^\prime)\right) \label{e_18.2}.
\end{align}
\endgroup
Here, \eqref{e_18.2} follows from noting that $E_{\bx \sim p_{\Psi, \btheta} } \bx = \bmu(\btheta)$.
\end{proof} 

\section{Supporting Lemmata}

\begin{lem}\label{lem_17.1}
   For any $\delta \le 1/e$, $\int_0^\delta \sqrt{\log(1/\epsilon)} d\epsilon \le 2 \delta \sqrt{\log(1/\delta)}$.
\end{lem}
\begin{proof}
    We start by making a transformation $x = \log(1/\epsilon)$ and observe that,
    \begingroup
    \allowdisplaybreaks
    \begin{align}
        \int_0^\delta \sqrt{\log(1/\epsilon)} d\epsilon = \int_{\log(1/\delta)}^\infty \sqrt{x} e^{-x} dx = & \int_{\log(1/\delta)}^\infty \sqrt{x} e^{-x/2} e^{-x/2} dx \nonumber \\
        \le & \sqrt{\log(1/\delta)} e^{-\frac{1}{2}\log(1/\delta)} \int_{\log(1/\delta)}^\infty e^{-x/2} dx \label{e112}\\
        =  & 2 \sqrt{\log(1/\delta)} e^{-\frac{1}{2}\log(1/\delta)} e^{-x/2}|_{\log(1/\delta)}^\infty \nonumber \\
        = & 2 \sqrt{\log(1/\delta)} e^{\log(1/\delta)} \nonumber \\
        = & 2 \delta \sqrt{\log(1/\delta)}. \nonumber
    \end{align}
    \endgroup
    In the above calculations, \eqref{e112} follows from the fact that the function $\sqrt{x} e^{-x/2}$ is decreasing when $x\ge 1$.
\end{proof}

\begin{lem}\label{lem20.1}
Let $Z \sim p_{\Psi, \theta}$, with $ \theta \in \Theta = \Real$. Then, $Z - \E Z$ is $\sigma_1$-SubGaussian.    
\end{lem}
\begin{proof}
We observe the following,
\begin{align*}
    \E e^{\lambda (Z - \E Z)} = e^{-\lambda \nabla \Psi(\theta)} \int e^{\lambda z} e^{\theta z - \Psi(\theta) } h(z) d\tau(z)
    & = e^{-\lambda \nabla \Psi(\theta) - \Psi(\theta)} \int  e^{(\theta+ \lambda) z } h(z) d\tau(z)\\
    & = e^{ \Psi(\theta + \lambda) - \Psi(\theta) - \lambda \nabla \Psi(\theta) } \le e^{\frac{ \sigma_1\lambda^2 }{2}}.
\end{align*}
\end{proof}

\begin{lem}\label{lem:1}
    Suppose that $Z_1, \dots, Z_n$ are independent and identically distributed sub-Gaussian random variables with variance proxy $\sigma^2$ and suppose that $\|f\|_\infty \le b$ for all $f \in \cF$. Then with probability at least $1-\delta$, 
    \[ \frac{1}{n}\sup_{f \in \cF} \sum_{i=1}^n Z_i f(\bx_i) - \frac{1}{n}\E\sup_{f \in \cF} \sum_{i=1}^n Z_i f(\bx_i) \precsim b \sigma \sqrt{\frac{\log(1/\delta)}{n}}.\]
\end{lem}
\begin{proof}
    Recall that for a random variable, $Z$, $\|Z\|_{\psi_2} = \sup_{p \ge 1} \frac{(\E |Z|^p)^{1/p}}{\sqrt{p}}$. Let $g(Z) = \frac{1}{n}\sup_{f \in \cF} \sum_{i=1}^n Z_i f(\bx_i)$. Using the notations of \cite{maurer2021concentration}, we note that 
    \begin{align}
        \|g_k(Z)\|_{\psi_2} = & \frac{1}{n} \left\|\sup_{f \in \cF} \left(\sum_{i\neq k} z_i f(\bx_i) + Z_k f(\bx_k) \right) - \E_{Z_k^\prime} \sup_{f \in \cF} \left(\sum_{i\neq k} z_i f(\bx_i) + Z_k^\prime f(\bx_k) \right) \right\|_{\psi_2} \nonumber\\
        \le & \frac{1}{n} \left\|\E_{Z_k^\prime}|Z_k-Z_k^\prime f(\bx_k)| \right\|_{\psi_2} \nonumber\\
        \le & \frac{b}{n} \left\|\E_{Z_k^\prime} |Z_k-Z_k^\prime| \right\|_{\psi_2} \label{se1}\\
        \le & \frac{b}{n} \left\|Z_k-Z_k^\prime \right\|_{\psi_2} \nonumber\\
        \le & \frac{2b}{n} \left\|Z_k\right\|_{\psi_2} \nonumber\\
        \precsim & \frac{b \sigma}{n} \nonumber.
    \end{align}
    Here, \eqref{se1} follows from \cite[Lemma 6]{maurer2021concentration}. Thus, $\left\|\sum_{k=1}^n\|g_k(Z)\|_{\psi_2}^2 \right\|_\infty \precsim b^2\sigma^2/n$.  Hence applying \citep[Theorem 3]{maurer2021concentration}, we note that with probability at least $1-\delta$, 
     \[ \frac{1}{n}\sup_{f \in \cF} \sum_{i=1}^n Z_i f(\bx_i) - \frac{1}{n}\E\sup_{f \in \cF} \sum_{i=1}^n Z_i f(\bx_i) \precsim b \sigma \sqrt{\frac{\log(1/\delta)}{n}}.\]
\end{proof}

\section{Supporting Results from the Literature}
\begin{lem}[Lemma B.1 of \cite{telgarsky2013moment}]\label{lem_tel}
     If differentiable $f$ is $r_1$ strongly convex, then $B_f(x\| y) \geq r_1\|x - y\|_2^2$. Furthermore, if differentiable $f$ has Lipschitz gradients with parameter $r_2$ with respect to $\| \cdot \|_2$, then $B_f(x\| y) \leq r_2\|x - y\|_2^2$.
 \end{lem}
\begin{lem}[Lemma 40 of \cite{chakraborty2024statistical}]\label{s.5}
   For any $m \ge \frac{1}{2} (\log_2 (4d) - 1)$, %\max\{\frac{1}{2} (\log_2 (4d) - 1), \log_2(2d)\}$, 
   we can construct a ReLU network $\text{prod}^{(d)}_m : \Real^d \to \Real$, such that for any $x_1, \dots, x_d \in [-1,1]$, $\| \text{prod}_m^{(d)}(x_1, \dots, x_d) - x_1\dots x_d\|_{\cL_\infty([-1,1]^d)} \le \frac{1}{2^{m}}$. Furthermore, for some absolute constant $c_3$,
   \begin{enumerate}
       \item $\cL(\text{prod}^{(d)}_m) \le c_3 m$, $\cW(\text{prod}^{(d)}_m) \le c_3 m$.
       \item $\cB(\text{prod}^{(d)}_m) \le 4 $.
   \end{enumerate}
\end{lem}

\section*{Acknowledgment}
We gratefully acknowledge the support of the NSF and the Simons Foundation for the Collaboration on the Theoretical Foundations of Deep Learning through awards DMS-2031883 and \#814639, the NSF’s support of FODSI through grant DMS-2023505, and the support of the ONR through MURI award N000142112431.
\bibliographystyle{apalike}
%\bibliography{mybib}

\begin{thebibliography}{}

\bibitem[Anthony and Bartlett, 1999]{anthony1999neural}
Anthony, M. and Bartlett, P. (1999).
\newblock {\em Neural network learning: Theoretical foundations}.
\newblock Cambridge University Press.

\bibitem[Banerjee et~al., 2005]{banerjee2005clustering}
Banerjee, A., Merugu, S., Dhillon, I.~S., Ghosh, J., and Lafferty, J. (2005).
\newblock Clustering with bregman divergences.
\newblock {\em Journal of machine learning research}, 6(10).

\bibitem[Bartlett et~al., 2019]{bartlett2019nearly}
Bartlett, P.~L., Harvey, N., Liaw, C., and Mehrabian, A. (2019).
\newblock Nearly-tight vc-dimension and pseudodimension bounds for piecewise
  linear neural networks.
\newblock {\em The Journal of Machine Learning Research}, 20(1):2285--2301.

\bibitem[Bousquet, 2002]{1444}
Bousquet, O. (2002).
\newblock {\em Concentration Inequalities and Empirical Processes Theory
  Applied to the Analysis of Learning Algorithms}.
\newblock PhD thesis, Biologische Kybernetik.

\bibitem[Chakraborty and Bartlett, 2024a]{chakraborty2024a}
Chakraborty, S. and Bartlett, P. (2024a).
\newblock A statistical analysis of wasserstein autoencoders for intrinsically
  low-dimensional data.
\newblock In {\em The Twelfth International Conference on Learning
  Representations}.

\bibitem[Chakraborty and Bartlett, 2024b]{chakraborty2024statistical}
Chakraborty, S. and Bartlett, P.~L. (2024b).
\newblock On the statistical properties of generative adversarial models for
  low intrinsic data dimension.
\newblock {\em arXiv preprint arXiv:2401.15801}.

\bibitem[Chen et~al., 2019]{chen2019efficient}
Chen, M., Jiang, H., Liao, W., and Zhao, T. (2019).
\newblock Efficient approximation of deep relu networks for functions on low
  dimensional manifolds.
\newblock {\em Advances in neural information processing systems}, 32.

\bibitem[Chen et~al., 2022]{chen2022nonparametric}
Chen, M., Jiang, H., Liao, W., and Zhao, T. (2022).
\newblock Nonparametric regression on low-dimensional manifolds using deep relu
  networks: Function approximation and statistical recovery.
\newblock {\em Information and Inference: A Journal of the IMA},
  11(4):1203--1253.

\bibitem[Cover and Thomas, 2005]{cover2005elements}
Cover, T.~M. and Thomas, J.~A. (2005).
\newblock {\em Elements of Information Theory}.
\newblock John Wiley \& Sons, Hoboken, NJ.

\bibitem[Cybenko, 1989]{cybenko1989approximation}
Cybenko, G. (1989).
\newblock Approximation by superpositions of a sigmoidal function.
\newblock {\em Mathematics of control, signals and systems}, 2(4):303--314.

\bibitem[Donahue et~al., 2017]{donahue2017adversarial}
Donahue, J., Kr{\"a}henb{\"u}hl, P., and Darrell, T. (2017).
\newblock Adversarial feature learning.
\newblock In {\em International Conference on Learning Representations}.

\bibitem[Dudley, 1969]{dudley1969speed}
Dudley, R.~M. (1969).
\newblock The speed of mean glivenko-cantelli convergence.
\newblock {\em The Annals of Mathematical Statistics}, 40(1):40--50.

\bibitem[Hornik, 1991]{hornik1991approximation}
Hornik, K. (1991).
\newblock Approximation capabilities of multilayer feedforward networks.
\newblock {\em Neural networks}, 4(2):251--257.

\bibitem[Huang et~al., 2022]{huangjmlr}
Huang, J., Jiao, Y., Li, Z., Liu, S., Wang, Y., and Yang, Y. (2022).
\newblock An error analysis of generative adversarial networks for learning
  distributions.
\newblock {\em Journal of Machine Learning Research}, 23(116):1--43.

\bibitem[Jiao et~al., 2021]{jiao2021deep}
Jiao, Y., Shen, G., Lin, Y., and Huang, J. (2021).
\newblock Deep nonparametric regression on approximately low-dimensional
  manifolds.
\newblock {\em arXiv preprint arXiv:2104.06708}.

\bibitem[Kakade et~al., 2009]{kakade2009duality}
Kakade, S., Shalev-Shwartz, S., Tewari, A., et~al. (2009).
\newblock On the duality of strong convexity and strong smoothness: Learning
  applications and matrix regularization.
\newblock {\em Unpublished Manuscript, http://ttic. uchicago.
  edu/shai/papers/KakadeShalevTewari09. pdf}, 2(1):35.

\bibitem[Kolmogorov and Tikhomirov, 1961]{kolmogorov1961}
Kolmogorov, A.~N. and Tikhomirov, V.~M. (1961).
\newblock $\epsilon$-entropy and $\epsilon$-capacity of sets in function
  spaces.
\newblock {\em Translations of the American Mathematical Society}, 17:277--364.

\bibitem[Lehmann and Casella, 2006]{lehmann2006theory}
Lehmann, E.~L. and Casella, G. (2006).
\newblock {\em Theory of Point Estimation}.
\newblock Springer Science \& Business Media.

\bibitem[Lu et~al., 2021]{lu2021deep}
Lu, J., Shen, Z., Yang, H., and Zhang, S. (2021).
\newblock Deep network approximation for smooth functions.
\newblock {\em SIAM Journal on Mathematical Analysis}, 53(5):5465--5506.

\bibitem[Maurer and Pontil, 2021]{maurer2021concentration}
Maurer, A. and Pontil, M. (2021).
\newblock Concentration inequalities under sub-gaussian and sub-exponential
  conditions.
\newblock {\em Advances in Neural Information Processing Systems},
  34:7588--7597.

\bibitem[Nakada and Imaizumi, 2020]{nakada}
Nakada, R. and Imaizumi, M. (2020).
\newblock Adaptive approximation and generalization of deep neural network with
  intrinsic dimensionality.
\newblock {\em Journal of Machine Learning Research}, 21(174):1--38.

\bibitem[Paul et~al., 2021]{paul2021uniform}
Paul, D., Chakraborty, S., Das, S., and Xu, J. (2021).
\newblock Uniform concentration bounds toward a unified framework for robust
  clustering.
\newblock {\em Advances in Neural Information Processing Systems},
  34:8307--8319.

\bibitem[Petersen and Voigtlaender, 2018]{petersen2018optimal}
Petersen, P. and Voigtlaender, F. (2018).
\newblock Optimal approximation of piecewise smooth functions using deep relu
  neural networks.
\newblock {\em Neural Networks}, 108:296--330.

\bibitem[Pope et~al., 2020]{pope2020intrinsic}
Pope, P., Zhu, C., Abdelkader, A., Goldblum, M., and Goldstein, T. (2020).
\newblock The intrinsic dimension of images and its impact on learning.
\newblock In {\em International Conference on Learning Representations}.

\bibitem[Posner et~al., 1967]{posner1967epsilon}
Posner, E.~C., Rodemich, E.~R., and Rumsey~Jr, H. (1967).
\newblock Epsilon entropy of stochastic processes.
\newblock {\em The Annals of Mathematical Statistics}, pages 1000--1020.

\bibitem[Schmidt-Hieber, 2020]{schmidt2020nonparametric}
Schmidt-Hieber, J. (2020).
\newblock {Nonparametric regression using deep neural networks with ReLU
  activation function}.
\newblock {\em The Annals of Statistics}, 48(4):1875 -- 1897.

\bibitem[Shen et~al., 2019]{shen2019nonlinear}
Shen, Z., Yang, H., and Zhang, S. (2019).
\newblock Nonlinear approximation via compositions.
\newblock {\em Neural Networks}, 119:74--84.

\bibitem[Shen et~al., 2022]{shen2022optimal}
Shen, Z., Yang, H., and Zhang, S. (2022).
\newblock Optimal approximation rate of relu networks in terms of width and
  depth.
\newblock {\em Journal de Math{\'e}matiques Pures et Appliqu{\'e}es},
  157:101--135.

\bibitem[Suzuki, 2018]{suzuki2018adaptivity}
Suzuki, T. (2018).
\newblock Adaptivity of deep relu network for learning in besov and mixed
  smooth besov spaces: optimal rate and curse of dimensionality.
\newblock {\em arXiv preprint arXiv:1810.08033}.

\bibitem[Suzuki and Nitanda, 2021]{suzuki2021deep}
Suzuki, T. and Nitanda, A. (2021).
\newblock Deep learning is adaptive to intrinsic dimensionality of model
  smoothness in anisotropic besov space.
\newblock {\em Advances in Neural Information Processing Systems},
  34:3609--3621.

\bibitem[Telgarsky and Dasgupta, 2013]{telgarsky2013moment}
Telgarsky, M.~J. and Dasgupta, S. (2013).
\newblock Moment-based uniform deviation bounds for $ k $-means and friends.
\newblock {\em Advances in Neural Information Processing Systems}, 26.

\bibitem[Tsybakov, 2009]{tsybakov2009introduction}
Tsybakov, A.~B. (2009).
\newblock {\em Introduction to Nonparametric Estimation}.
\newblock Springer Series in Statistics. Springer, Springer New York, NY, 1
  edition.
\newblock Published: 26 November 2008.

\bibitem[Uppal et~al., 2019]{uppal2019nonparametric}
Uppal, A., Singh, S., and P{\'o}czos, B. (2019).
\newblock Nonparametric density estimation \& convergence rates for gans under
  besov ipm losses.
\newblock {\em Advances in neural information processing systems}, 32.

\bibitem[Vershynin, 2018]{vershynin2018high}
Vershynin, R. (2018).
\newblock {\em High-dimensional probability: An introduction with applications
  in data science}, volume~47.
\newblock Cambridge university press.

\bibitem[Wainwright, 2019]{wainwright_2019}
Wainwright, M.~J. (2019).
\newblock {\em High-Dimensional Statistics: A Non-Asymptotic Viewpoint}.
\newblock Cambridge Series in Statistical and Probabilistic Mathematics.
  Cambridge University Press.

\bibitem[Weed and Bach, 2019]{weed2019sharp}
Weed, J. and Bach, F. (2019).
\newblock Sharp asymptotic and finite-sample rates of convergence of empirical
  measures in wasserstein distance.
\newblock {\em Bernoulli}, 25(4A):2620 -- 2648.

\bibitem[Yang and Barron, 1999]{yang1999information}
Yang, Y. and Barron, A. (1999).
\newblock Information-theoretic determination of minimax rates of convergence.
\newblock {\em Annals of Statistics}, pages 1564--1599.

\bibitem[Yarotsky, 2017]{yarotsky2017error}
Yarotsky, D. (2017).
\newblock Error bounds for approximations with deep relu networks.
\newblock {\em Neural Networks}, 94:103--114.

\end{thebibliography}

\end{document}